\documentclass{article}

\usepackage{microtype}
\usepackage{graphicx}
\usepackage{subcaption}
\usepackage{booktabs} 
\usepackage{multirow}
\usepackage{makecell}

\usepackage{thm-restate}

\usepackage{hyperref}

\usepackage[accepted]{icml2026}

\usepackage{amsmath}
\usepackage{amssymb}
\usepackage{mathtools}
\usepackage{amsthm}

\usepackage{xcolor}
\usepackage{amssymb}
\usepackage{pifont}

\newcommand{\bluecheck}{{\color{teal}\checkmark}}
\newcommand{\xmark}{\color{red}\ding{55}}
\newcommand{\restatedtitle}{}
\newtheorem*{restatedlemma}{\restatedtitle}

\newenvironment{restatelemma}[2]{
  \renewcommand{\restatedtitle}{Lemma~\ref{#1} (\textnormal{#2})}%
  \begin{restatedlemma}
}{\end{restatedlemma}}

\newcommand{\restatedproptitle}{}
\newtheorem*{restatedproposition}{\restatedproptitle}

\newenvironment{restateproposition}[2]{
  \renewcommand{\restatedproptitle}{Proposition~\ref{#1} (\textnormal{#2})}%
  \begin{restatedproposition}
}{\end{restatedproposition}}

\newcommand{\restatedtheoremtitle}{}
\newtheorem*{restatedtheorem}{\restatedtheoremtitle}

\newcommand{\maybeparen}[1]{
  \if\relax\detokenize{#1}\relax
  \else
    \ (\textnormal{#1})
  \fi
}

\newenvironment{restatetheorem}[2]{
  \renewcommand{\restatedtheoremtitle}{Theorem~\ref{#1}\maybeparen{#2}}%
  \begin{restatedtheorem}
}{\end{restatedtheorem}}

\usepackage[capitalize,noabbrev]{cleveref}

\theoremstyle{plain}
\newtheorem{theorem}{Theorem}[section]
\newtheorem{proposition}[theorem]{Proposition}
\newtheorem{lemma}[theorem]{Lemma}

\theoremstyle{definition}

\theoremstyle{remark}
\newtheorem{remark}[theorem]{Remark}

\usepackage[textsize=tiny]{todonotes}

\icmltitlerunning{Unbiased and Second-Order-Free Training for High-Dimensional PDEs}

\begin{document}

\twocolumn[
  \icmltitle{Unbiased and Second-Order-Free Training for High-Dimensional PDEs}



  \icmlsetsymbol{equal}{*}

  \begin{icmlauthorlist}
    \icmlauthor{Jaemin Seo}{cau}
    \icmlauthor{Surin Lee}{cau}
    \icmlauthor{Jae Yong Lee}{cau}
  \end{icmlauthorlist}

  \icmlaffiliation{cau}{Department of Artificial Intelligence, Chung-Ang University, Seoul, Republic of Korea}
  
  \icmlcorrespondingauthor{Jae Yong Lee}{jaeyong@cau.ac.kr}

  \icmlkeywords{Machine Learning, ICML}

  \vskip 0.3in
]



\printAffiliationsAndNotice{}  

\begin{abstract}
  Deep learning methods based on backward stochastic differential equations (BSDEs) have emerged as competitive alternatives to physics-informed neural networks (PINNs) for solving high-dimensional partial differential equations (PDEs). By leveraging probabilistic representations, BSDE approaches can avoid the curse of dimensionality and often admit second-order-free training objectives that do not require explicit Hessian evaluations. It has recently been established that the commonly used Euler–Maruyama (EM) time discretization induces an intrinsic bias in BSDE training losses. While high-order schemes such as Heun can fully eliminate this bias, such schemes re-introduce second-order spatial derivatives and incur substantial computational overhead. In this work, we provide a principled analysis of EM-induced loss bias and propose an unbiased, second-order-free training framework that preserves the computational advantages of BSDE methods. Our code is available at \url{https://github.com/seojaemin22/Un-EM-BSDE}.
\end{abstract}

\section{Introduction}

The Deep Backward Stochastic Differential Equation (BSDE) method \citep{raissi2024forward, HanJentzenE2018PNAS} serves as a critical complementary approach to Physics-Informed Neural Networks (PINNs) research \citep{raissi2019physics} by addressing the curse of dimensionality in high-dimensional PDEs. While standard PINN embeds the PDE residual in its loss function, the Deep BSDE method leverages the well-known mathematical connection between PDEs and stochastic differential equations (SDEs), reformulating the problem to enable efficient stochastic sampling along paths instead of relying on high-dimensional grid discretization, thereby offering a scalable and robust alternative for complex high-dimensional systems.

The BSDE methodology for solving high-dimensional PDEs approximates the solution by discretizing the paths of SDEs \citep{han2025brief}. This process inherently requires enforcing temporal consistency through a self-consistency loss, where the solution values at two consecutive time steps are trained to satisfy the discretized BSDE \citep{nusken2021interpolating}. Typically, the BSDE is discretized in time using the Euler–Maruyama (EM) integration scheme. This time discretization yields the so-called EM-BSDE method, which is then solved by imposing a self-consistency loss. However, this class of solvers suffers from a significant discretization-induced bias issue, which severely degrades optimization performance. As a solution, \citet{park2025integration} proposed a new Heun-BSDE method, combining the Stratonovich SDE formulation with the Stochastic Heun integration scheme. The Heun-BSDE approach eliminates the discretization-induced bias inherent in the conventional EM-BSDE method, leading to a substantial improvement in the accuracy of BSDE-based solvers.

\begin{table}[t]
\caption{Summary of bias properties, computational requirements, and training time.}
\label{training_time}
\begin{center}
\begin{small}
\resizebox{\columnwidth}{!}{
\begin{tabular}{lccc}
\toprule
\textbf{Method} & \textbf{Unbiased} & \textbf{2nd-order-free} & \textbf{Time} \\
\midrule
 EM-BSDE \citep{raissi2024forward} & \xmark & \bluecheck  & 1x \\
 Shotgun \citep{xu2025deep} & \xmark & \bluecheck & 0.75x  \\
 Multi-Shot EM-BSDE & \xmark & \bluecheck & 1.74x  \\
 Heun-BSDE \citep{park2025integration} & \bluecheck & \xmark & 42.91x  \\
 FS-PINNs \citep{park2025integration} & \bluecheck & \xmark & 32.07x  \\
 \midrule
 \textbf{Un-EM-BSDE (ours)} & \bluecheck & \bluecheck & \textbf{1.79x} \\
\bottomrule
\end{tabular}
}
\end{small}
\end{center}
\vskip -0.1in
\end{table}

However, the performance of the Heun-BSDE method is limited by two main issues. Firstly, the Heun-BSDE method fails to demonstrate a clear enhancement compared to the FS-PINNs (Forward SDE-based Physics-Informed Neural Networks) method, which minimizes the standard PINN loss using sampled data along the identical forward SDE trajectory. The primary reason for this limited improvement is that the Heun-BSDE approach fundamentally still involves minimizing an estimate of the squared PDE residual. Secondly, Heun-BSDE is significantly more expensive than the baseline EM-BSDE and also incurs a higher training time compared to FS-PINNs. As reported in Table~\ref{training_time}, Heun-BSDE requires $42.91\times$ the training time of EM-BSDE, compared to $32.07\times$ for FS-PINNs. This increased computational cost stems from the need for multiple second-order derivative calculations. Specifically, while FS-PINNs require computing the Hessian only once per PDE residual calculation, the Heun-BSDE method requires two Hessian computations for a single step within the stochastic Heun scheme. Consequently, given the lack of substantial performance gains and the increased computational burden relative to FS-PINNs, the Heun-BSDE method presents limited practical advantage.

In this paper, we present a fundamental solution to the discretization-induced bias problem, avoiding the limited empirical gains and the heavy second-order-derivative overhead, by introducing a randomized unbiased loss estimator formed as the product of two independent one-step errors:

\begin{itemize}
    \item We propose an \textbf{Un}biased \textbf{EM-BSDE} method (\textbf{Un-EM-BSDE}), which enables training free of bias. While motivated by the general philosophy of combining complementary objectives in \citet{hu2025bias}, our method introduces a novel loss formulation specifically designed to eliminate bias induced by discretization in EM-BSDE training.
    \item We theoretically demonstrate that the proposed no-bias estimator shows no significant difference in terms of variance compared to the conventional biased EM-BSDE method. Furthermore, we suggest the potential to more efficiently control variance and enhance solver stability by leveraging a Shotgun-inspired training strategy \citep{xu2025deep}.
    \item We experimentally confirm that the proposed method achieves the accuracy level of the Heun-BSDE and FS-PINNs, even though its training time remains similar to the EM-BSDE, as it does not require direct computation of the second-order derivative term. Furthermore, we will demonstrate that any biased method can be converted into an unbiased one via the proposed debiasing mechanism, thereby yielding improved performance.
\end{itemize}

\section{Related Work}

\noindent\textbf{Physics-Informed Neural Networks}. PINNs have recently gained significant attention as a powerful, mesh-free methodology for solving high-dimensional PDEs \citep{raissi2019physics, sirignano2018dgm, karniadakis2021physics, hu2024tackling}.
PINNs parameterize the solution of the PDE using a neural network and directly minimize the residual of the PDE along with the boundary and initial condition residuals as a loss function through gradient descent.
This approach has demonstrated impressive results across a wide range of problems \citep{raissi2018hidden, kissas2020machine}. However, PINNs often suffer from trainability and accuracy difficulties for high-frequency or multiscale solutions \citep{wang2022and, rathore2024challenges, wang2022is}.
Stable and accurate high-dimensional PINN solvers therefore remain an active research direction.

\noindent\textbf{Deep BSDE method}.
On the other hand, methodologies based on BSDEs are presented as a complementary approach for solving high-dimensional PDEs \citep{zhang2026deep}. 
These methods seek to overcome the curse of dimensionality by recasting the PDE as a forward-backward SDE and deriving a trajectory-based loss.
Several works, including \citet{han2017deep,HanJentzenE2018PNAS,HurePhamWarin2020MComp, BeckEtAl2021SISCDeepSplitting, GermainPhamWarin2022SISC}, and \citet{NeufeldSchmockerWu2025CNSNS}, utilize separate networks to predict the solution and gradient of each time step.
In contrast, \citet{zhang2022fbsde, raissi2024forward}, \citet{xu2025deep}, and \citet{doi:10.1137/24M169312X} employ a unified model, thereby achieving a better representation of the gradient continuity across steps and enabling the retrieval of a visual trajectory.
This BSDE-based approach offers different training dynamics compared to the optimization issues faced by PINNs. Therefore, research aimed at understanding the performance differences between the two methodologies is a crucial research direction in the development of high-dimensional PDE solvers.

\noindent\textbf{Randomized estimation for high-order derivatives}. 
To improve the computational efficiency of PINNs in high-dimensional PDE problems, stochastic estimation methodologies \citep{hutchinson1989stochastic, bekas2007estimator} that reduce the cost of high-order differentiation are being actively studied.
Among these, a representative approach is to incorporate stochastic estimation within Automatic Differentiation to obtain low-cost approximations of higher-order derivatives \citep{oktay2021randomized, shi2024stochastic}, or to modify the network architecture and optimization scheme so that stochastic estimation becomes an integral component of the PINN training process \citep{hu2024hutchinson, hu2024tackling, hu2025bias}. Our method employs a different form of stochastic estimation to eliminate the bias.

\section{Preliminary}

\subsection{Problem setup}

We study learning-based approximation of the solution $u : \mathcal{T} \times \Omega \to \mathbb{R}$, where $\mathcal{T} = [0, T_{\text{end}}]$ and $\Omega \subset \mathbb{R}^d$, for non-linear terminal-value PDEs. We write the PDE in operator form
\begin{equation}
    \label{PDE}
    \mathcal{L}[u](t,x) = \phi(t, x, u(t, x), \nabla u(t, x)),
\end{equation}
where
\begin{multline}
    \mathcal{L}[u](t,x) :=\partial_t u(t,x) + \mu(t,x) \cdot \nabla u(t,x) \\
    + \frac{1}{2} \mathrm{Tr}[\sigma^T (\nabla^2 u) \sigma](t,x),
\end{multline}
for $(t,x)\in \mathcal{T} \times \Omega$. We impose the terminal condition $u(T_{\text{end}}, x) = g(x)$ for $x\in \Omega$. We denote by $\nabla u$ and $\nabla^2 u$ the spatial gradient and Hessian, respectively. Here $\mu : \mathcal{T} \times \Omega \to \mathbb{R}^d$ and $\sigma : \mathcal{T} \times \Omega \to \mathbb{R}^{d\times d}$ represent the drift and diffusion coefficients, and $\phi : \mathcal{T} \times \Omega \times \mathbb{R} \times \mathbb{R}^d \to \mathbb{R}$ captures the nonlinear terms depending on $(u, \nabla u)$.

Our objective is to learn an accurate approximation of the PDE solution within a neural-network hypothesis class. Specifically, we parameterize $u_\theta : \mathcal{T} \times \Omega \to \mathbb{R}$ where $u_\theta$ is a neural network with parameters $\theta \in \Theta$, and seek $\theta^*$ such that $u_{\theta^*}$ satisfies the terminal condition and approximately solves the operator equation $\mathcal{L}[u] = \phi$ over $\mathcal{T} \times \Omega$. 
In many applications, however, the primary quantity of interest is the initial-time value $u(0, x_0)$ for a prescribed initial state $x_0$.
This setting is common in BSDE-based high-dimensional terminal-value solvers, which often target $u(0,x_0)$ rather than a uniform approximation of the full solution field.
Accordingly, we focus on training objectives and solver designs that are tailored to accurately recovering this initial-time quantity. 
This perspective also clarifies why we do not adopt standard collocation-based PINN formulations: enforcing the PDE uniformly across the entire spatio-temporal domain typically requires dense sampling in $(t, x)$, which can be computationally demanding and often difficult to optimize, especially as the dimension $d$ grows.

A natural alternative is to avoid global domain-wide enforcement and instead impose the governing equation along representative trajectories of an underlying stochastic process. Specifically, by Itô's formula, one can verify that the solution $u$ of $\eqref{PDE}$ is associated with the following forward-backward stochastic differential equation (FBSDE) system
\begin{equation}
    \label{FBSDE}
    \begin{cases}
        dX_t = \mu(t, X_t) dt + \sigma(t, X_t) dW_t, \\
        dY_t = \phi(t, X_t, Y_t, Z_t) dt + Z_t^T \sigma(t, X_t) dW_t,
    \end{cases}
\end{equation}
for $t \in \mathcal{T}$, where $X_0 = x_0$, $Y_{T_{\text{end}}} = g(X_{T_{\text{end}}})$, and $\{W_t : t \in \mathcal{T}\}$ is a $d$-dimensional standard Brownian motion. In particular, the correspondence is given by
\begin{equation}
    Y_t = u(t, X_t), ~~~~ Z_t = \nabla u(t, X_t).
\end{equation}

In the continuous-time formulation, the training objective is expressed through quantities defined on $\mathcal{T}$ and typically involves expectations and time-integrated terms along the stochastic dynamics. In practice, these continuous-time objects are evaluated via numerical approximation, which motivates introducing a finite set of time points to approximate the underlying integrals. Throughout this work, we adopt a standard time-stepping approximation: we select an integer $N$ and consider the grid $t_n := n\Delta t$ with $\Delta t := T_{\text{end}}/N$. This choice is made for clarity of exposition and computational convenience, and the same construction extends in a straightforward manner to non-uniform grids or alternative quadrature/time-sampling schemes.

Correspondingly, we work with Brownian increments $\Delta W_n := W_{t_{n+1}} - W_{t_n} \sim \mathcal{N}(0, \Delta t I_d)$ on the chosen grid and represent the dynamics through step-wise update maps. Specifically, we discretize the FBSDE using the Euler–Maruyama scheme and apply the corresponding one-step update
\begin{align}
    &F_n(x) := x + \mu(t_n, x) \Delta t + \sigma(t_n, x) \Delta W_n, \\
    &\begin{aligned}
        B_n(x; u) :=~& u(t_n, x) + \phi_u(t_n, x) \Delta t  \\
        & + \nabla u(t_n, x)^T \sigma(t_n, x) \Delta W_n,
    \end{aligned}
\end{align}
where $\phi_u(t, x) := \phi(t, x, u(t, x), \nabla u(t, x))$. Using $F_n$ and $B_n$, we introduce a one-step self-consistency error for the Euler-Maruyama discretization:
\begin{equation}
    \mathrm{err}_n^{\mathrm{EM}}(x; u) := \frac{u(t_{n+1}, F_n(x)) - B_n(x; u)}{\Delta t} .
\end{equation}
The corresponding one-step EM-BSDE loss is defined by:
\begin{equation}
    \ell_{\mathrm{EM}}(\theta; n, x) := \mathbb{E}\left[|\mathrm{err}_n^{\mathrm{EM}}(x; u_\theta)|^2 \right],
\end{equation}
where the expectation is taken with respect to $\Delta W_n$. Aggregating along the forward trajectory $\hat{X}_{n+1} = F_n(\hat{X}_n)$ with $\hat{X}_0 = x_0$, we obtain the EM-BSDE objective:
\begin{equation}
    L_{\mathrm{EM}}(\theta) := \frac{1}{N}\sum_{n=0}^{N-1}\mathbb{E}_{\hat{X}_n}\left[\ell_{\mathrm{EM}}(\theta; n, \hat{X}_n) \right].
\end{equation}
Based on this construction, \citet{raissi2024forward} approximate the solution of \eqref{FBSDE} by minimizing $L_{\mathrm{EM}}(\theta)$ over the network parameters $\theta$.

\subsection{Bias of the BSDE loss}

Recent results have shown that one-step self-consistency losses, when constructed from time-discretized BSDE updates, can exhibit a \emph{discretization-induced bias} at finite step size $\Delta t$, in the sense that they are biased estimators of the corresponding continuous-time self-consistency objective implied by the PDE/FBSDE correspondence. In particular, the EM-BSDE loss is generally biased.

\begin{lemma}[Bias of the EM-BSDE loss \citep{park2025integration}]
\label{bias_EM}
Suppose that $\mu, \sigma$ are bounded and $u_\theta$ is $C^{1,2}$. Then we have:
\begin{equation}
    \begin{aligned}
        \ell_{\mathrm{EM}}(\theta; n, x) =~& 
        \left( [\mathcal{L}[u_\theta] - \phi_{u_\theta}](t_n, x) \right)^2 \\
        &+ \frac{1}{2}\mathrm{Tr}[(\sigma^T(\nabla^2 u_\theta) \sigma)^2](t_n, x)  \\
        &+ O(\Delta t^{1/2}) .
    \end{aligned}
\end{equation}
\end{lemma}

One way to address this issue is to modify the underlying stochastic calculus used in the formulation. Specifically, \citet{park2025integration} proposes rewriting the dynamics in Stratonovich form and adopting a stochastic Heun integrator, which yields an alternative one-step update and the corresponding self-consistency error. The resulting Heun-BSDE loss is bias-free. See Appendix~\ref{Appendix_DeepBSDE} for the explicit definition of $\ell_{\mathrm{Heun}}$.

\begin{lemma}[Unbiasedness of the Heun-BSDE loss \citep{park2025integration}]
\label{bias_Heun}
Suppose that $\mu, \sigma$, and $\phi_{u_\theta}$ are all in $C^{1,1}$, and $u_\theta$ is in $C^{1,3}$. Then we have:
\begin{equation}
    \begin{aligned}
        \ell_{\mathrm{Heun}}(\theta; n, x) =~& \left( [\mathcal{L}[u_\theta] - \phi_{u_\theta}](t_n, x) \right)^2 \\
        &+ O(\Delta t^{1/2}) .
    \end{aligned}
\end{equation}
\end{lemma}

However, the Stratonovich-based Heun construction comes with a practical limitation. The Itô-to-Stratonovich conversion introduces correction terms that involve second-order spatial derivatives, so implementing the Heun-BSDE loss may require evaluating $\nabla^2 u_\theta$. This substantially increases the per-iteration cost and makes it comparable to PDE-residual-based methods.

Another way to address this issue is to keep the same one-step time-stepping scheme, but to reduce the discretization-induced bias by leveraging multiple independent Monte Carlo samples at each time step. In the Shotgun construction, we evaluate the one-step self-consistency error using $M$ i.i.d.\ Brownian increments and aggregate them into a single averaged error \citep{xu2025deep}. To formalize this, for any random quantity $\xi$, we define the Shotgun averaging operator as the empirical average of $M$ i.i.d.\ replicates:
\begin{equation}
    \mathrm{Shot}_M[\xi] := \frac{1}{M}\sum_{m=1}^{M} \xi_m ,
\end{equation}
where $\xi_1,\ldots,\xi_M \stackrel{\mathrm{i.i.d.}}{\sim} \xi$. Using this operator, we now define the one-step Shotgun loss by
\begin{equation}
    \ell^M_{\mathrm{SG}}(\theta; n, x) := \mathbb{E}\left[\left| \mathrm{Shot}_M\left[\mathrm{err}_n^{\mathrm{SG}}(x; u_\theta)\right] \right|^2 \right] ,
\end{equation}
where $\mathrm{err}_n^{\mathrm{SG}}(x; u_\theta)$ denotes the Brownian-increment-free one-step self-consistency error (see Appendix~\ref{Appendix_DeepBSDE} for the precise definition), and the expectation is taken with respect to the $M$ i.i.d.\ Brownian increments used in $\operatorname{Shot}_M[\cdot]$.

This multi-shot averaging leads to a reduced discretization-induced bias. We analyze the Shotgun construction of \citet{xu2025deep} and prove the following bias expansion: compared to the standard EM-BSDE loss, the bias of $\ell_{\mathrm{SG}}^M$ scales down by a factor of $1/M$.

\begin{proposition}[Bias of the Shotgun loss]
\label{bias_shotgun}
Suppose that $\mu, \sigma$ are bounded and $u_\theta$ is $C^{1,2}$. Let $\tau>0$ denote the time increment used in the Shotgun construction. Then we have:
\begin{equation}
    \begin{aligned}
        \ell^M_{\mathrm{SG}}(\theta; n, x) =~& \left( [\mathcal{L}[u_\theta] - \phi_{u_\theta}](t_n, x) \right)^2 \\
        &+ \frac{1}{2M}\mathrm{Tr}[(\sigma^T(\nabla^2 u_\theta) \sigma)^2](t_n, x)  \\
        &+ O(\tau) .
    \end{aligned}
\end{equation}
\end{proposition}

Similarly, we can apply the same multi-shot averaging operator to the standard EM one-step error. This yields the Multi-Shot EM one-step loss.

\begin{equation}
    \ell^M_{\mathrm{SEM}}(\theta; n, x) := \mathbb{E}\left[\left| \mathrm{Shot}_M\left[\mathrm{err}_n^{\mathrm{EM}}(x; u_\theta)\right] \right|^2 \right] .
\end{equation}

This Multi-Shot EM loss exhibits the same bias reduction as the Shotgun loss.
We show that this bias is reduced by a factor of $1/M$ in the following proposition.

\begin{proposition}[Bias of the Multi-Shot EM loss]
\label{bias_multishot_EM}
Suppose that $\mu, \sigma$ are bounded and $u_\theta$ is $C^{1,2}$. Then we have:
\begin{equation}
    \begin{aligned}
        \ell^M_{\mathrm{SEM}}(\theta; n, x) =~& \left( [\mathcal{L}[u_\theta] - \phi_{u_\theta}](t_n, x) \right)^2 \\
        &+ \frac{1}{2M}\mathrm{Tr}[(\sigma^T(\nabla^2 u_\theta) \sigma)^2](t_n, x)  \\
        &+ O(\Delta t^{1/2}) .
    \end{aligned}
\end{equation}
\end{proposition}

The proof of Proposition~\ref{bias_shotgun} and Proposition~\ref{bias_multishot_EM} are presented in Appendix~\ref{proof_existing_BSDE_bias}

Overall, existing approaches exhibit a clear trade-off between bias and efficiency. Heun-BSDE removes the discretization-induced bias at the one-step level, but the resulting objective is significantly more expensive to evaluate. The Shotgun and Multi-Shot constructions offer a cheaper mitigation by averaging over $M$ independent samples, reducing the bias roughly at a $1/M$ rate, but they do not eliminate it. These limitations motivate an alternative debiasing strategy that preserves the efficiency of EM-type updates while eliminating the discretization-induced bias.

\section{Proposed Method}

In this section, we propose an Un-EM-BSDE loss that removes the discretization-induced bias while keeping the computational cost comparable to the standard EM-BSDE objective. We also show that the resulting estimator has a variance comparable to that of the original EM-BSDE loss.

\begin{figure*}[t]
  \centering
  \includegraphics[width=1.8\columnwidth]{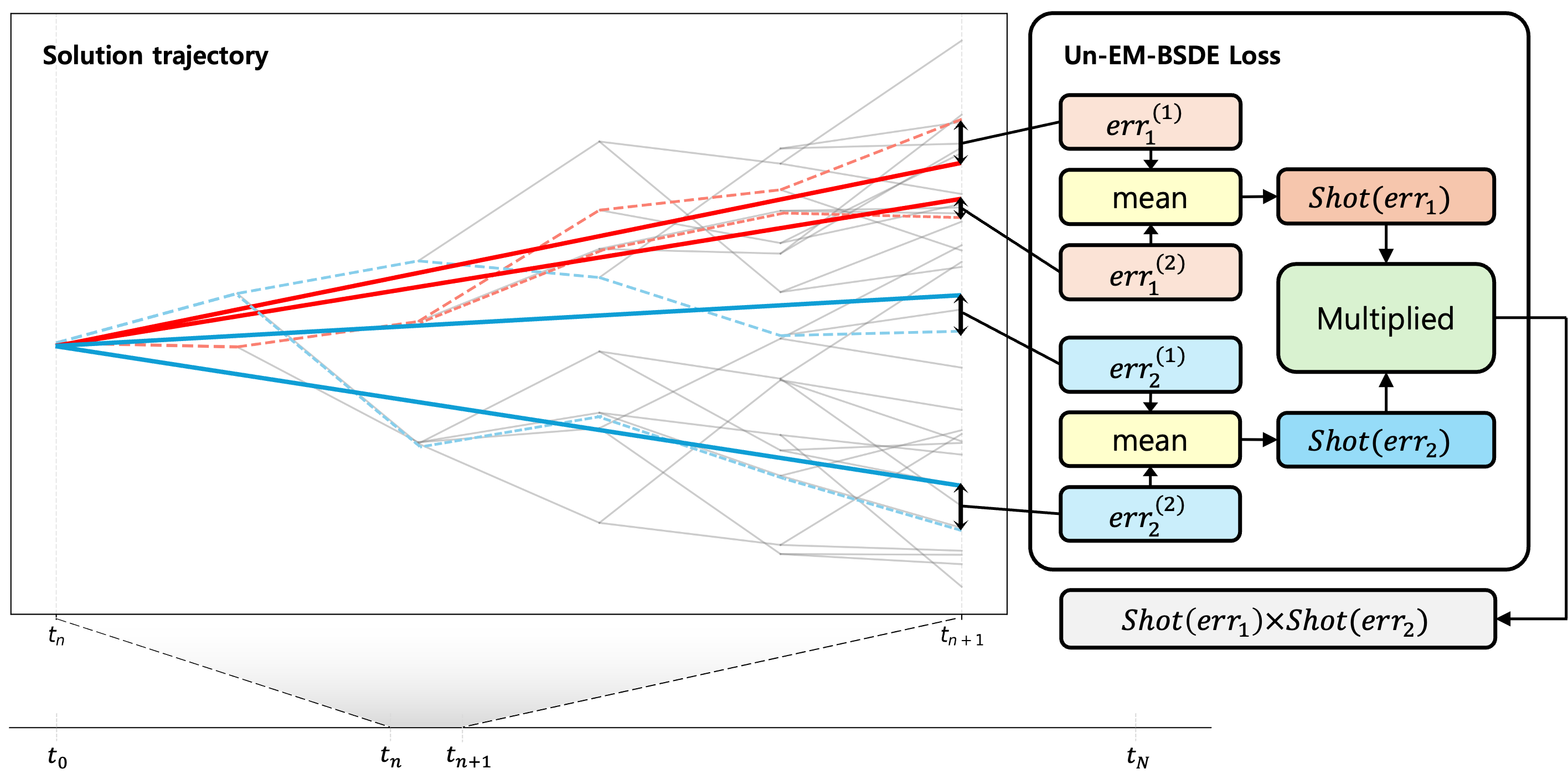}
  \caption{Illustration of the one-step loss computation for Un-EM-BSDE with $M_1 = M_2 = 2$. The many gray curves visualize the randomness of solution trajectories (possible SDE sample paths). The dashed curves denote a ground-truth trajectory, whereas the solid curves indicate the one-step predicted trajectory. The errors from the two independently formed groups are multiplied to produce an unbiased estimator.}
\end{figure*}

\subsection{Un-EM-BSDE}

Motivated by a bias-variance decomposition perspective and inspired by related constructions in \citet{hu2025bias}, we propose the following Un-EM-BSDE objective:

\begin{equation}
    \begin{aligned}
        \ell^{M_1, M_2}_{\mathrm{UEM}}(\theta, n, x) = \mathbb{E}[& ( \mathrm{Shot}_{M_1}[\mathrm{err}_n^{\mathrm{EM}}(x; u_\theta)]) \\
        &\times (\mathrm{Shot}_{M_2}[\mathrm{err}_n^{\mathrm{EM}}(x; u_\theta)])] .
    \end{aligned}
\end{equation}

The key idea is to form the objective using independent draws of the one-step error, so that the expectation is factorized. 
This is closely related to the classical sample-splitting principle in statistics, where independent samples are assigned to different factors so that expectations behave cleanly. In our setting, this principle is incorporated into the EM-BSDE objective in a problem-specific way to eliminate discretization-induced bias while preserving the second-order-free computational structure.

\begin{lemma}[Unbiasedness of the Un-EM-BSDE loss]
\label{bias_unbiased_EM}
Suppose that $\mu, \sigma$ are bounded and $u_\theta$ is $C^{1,2}$. Then we have
\begin{equation}
    \begin{aligned}
        \ell^{M_1, M_2}_{\mathrm{UEM}}(\theta; n, x) =~& \left(\mathcal{L}[u_\theta](t_n, x) - \phi_{u_\theta}(t_n, x)\right)^2 \\
        &+ O(\Delta t^{1/2}) .
    \end{aligned}
\end{equation}
\end{lemma}

Having established the unbiasedness of the proposed construction, we further note that the resulting discrete-time objective is also consistent with its continuous-time counterpart in the limit $\Delta t \to 0$, analogously to the standard EM-BSDE and Heun-BSDE objectives.

\begin{theorem}[Consistency of the Un-EM-BSDE loss]
\label{consistency_unbiased_EM}
Suppose that $f, g, \phi_{u_\theta} \in C^{1,0}$, $u_\theta \in C^{1,2}$, and $\Delta t \le 1$. We have
\begin{equation}
    \begin{aligned}
        &\frac{1}{N}\sum_{n=0}^{N-1} \mathbb{E}_{\hat{X}_n}[\ell_{\mathrm{UEM}}(\theta; n, \hat{X}_n)] \\
        =~& \frac{1}{|\mathcal{T}|} \int_\mathcal{T} \mathbb{E}\left[ \left( [\mathcal{L}[u_\theta] - \phi_{u_\theta}](t, X_t) \right)^2 \right] dt \\
        &+ O(\Delta t^{1/2}) .
    \end{aligned}
\end{equation}
\end{theorem}

The proofs of Lemma \ref{bias_unbiased_EM} and Theorem \ref{consistency_unbiased_EM} are provided in the appendix~\ref{proof_Un-EM-BSDE}.

\begin{algorithm*}[tb]
    \caption{Un-EM-BSDE loss algorithm}
    \label{Unbiased_EM-BSDE_algorithm}
{\small
\begin{algorithmic}
    \STATE {\bfseries Input:} batch size $B$, terminal time $T_{\text{end}}$, number of steps $N$, neural network $u_\theta(t, x)$, number of shots $M_1, M_2$
    \STATE Construct time grid: Define $\Delta t = T_{\text{end}}/N$, and define $T \in \mathbb{R}^{B \times (N+1)}$ by $T[b, n] = n\Delta t$ for each $b, n$
    \STATE Initialize candidate states: Allocate $X \in \mathbb{R}^{B \times (N+1) \times (M_1+M_2) \times d}$, and set $X[b, 0, 0] = x_0$ for each $b$
    \STATE Initialize main forward states: $(t, x) = (T[:, 0], X[:, 0, 0])$
    \FOR{$n=0$ {\bfseries to} $N-1$}
        \STATE Sample i.i.d. Brownian increments: $\{\Delta W_{n,i}^{(b)}\}_{b\in[B],\,i\in[M_1+M_2]} \overset{\text{i.i.d.}}{\sim}\mathcal{N}(0,\Delta t\, I_d)$.
        \STATE Construct candidate next states: $X[:, n+1, i] = x + \mu(t, x)\Delta t + \sigma(t, x)\Delta W_{n,i}^{(b)}$ for each $n,i$
        \STATE Update main forward states: $(t, x) = (T[:, n+1], X[:, n+1, 0])$
    \ENDFOR
    \STATE Construct main forward trajectory: Define $X_{\text{main}} = X[:, :, 0]$
    \STATE Evaluate network outputs: Define $Y \in \mathbb{R}^{B\times (N+1) \times (M_1 + M_2)}$ by $Y[b, n, i] = u_\theta(T[b, n], X[b, n, i])$ for each $b, n, i$
    \STATE Evaluate network gradients: Define $Z_{\text{main}} \in \mathbb{R}^{B \times (N+1) \times d}$ by $Z_{\text{main}}[b, n] = \nabla_x u_\theta(T, X_{\text{main}})[b, n]$ for each $b, n$
    \STATE Compute one-step BSDE prediction: Define $Y_{\text{main}} = Y[:, :, 0]$, and define $\hat{Y} \in \mathbb{R}^{B \times (N+1) \times (M_1 + M_2)}$ by:
    $$
    \hat{Y}[b, n+1, i] = Y_{\text{main}}[b, n] + \phi(T, X_{\text{main}}, Y_{\text{main}}, Z_{\text{main}})[b, n] \Delta t + [Z_{\text{main}}^T \sigma(T, X_{\text{main}})][b, n]\Delta W_{n,i}^{(b)} ~~~~ \text{for each }b, n, i
    $$
    \STATE Compute the per-sample step error and the group-wise averages: $\mathrm{err}_{n, i}^{(b)} = Y[b, n+1, i] - \hat{Y}[b, n+1, i]$ for each $b,n,i$, and 
    \begin{equation*}
        \mathrm{Shot\_err}^{(b)}_{n, 1} = \frac{1}{M_1}\sum_{i=0}^{M_1 - 1}\mathrm{err}_{n,i}^{(b)}, ~~~~ \mathrm{Shot\_err}^{(b)}_{n,2} = \frac{1}{M_2}\sum_{i=0}^{M_2 - 1} \mathrm{err}^{(b)}_{n, M_1+i} ~~~~ \text{for each}~b
    \end{equation*} 
    Compute BSDE self-consistency loss: $L_{\mathrm{bsde}} = \frac{1}{B}\sum_{b=0}^{B-1}\sum_{n=0}^{N-1}\left( \mathrm{Shot\_err}_{n,1}^{(b)} \times \mathrm{Shot\_err}_{n,2}^{(b)} \right)$
    \IF{we use a soft constraint}   
        \STATE Define $L_{T} = \frac{1}{B}\sum_{b=0}^{B-1}\left[(Y^{(b)}_N - g(T^{(b)}_N, X^{(b)}_N))^2 + \|Z^{(b)}_N - \nabla_x g(T^{(b)}_N, X^{(b)}_N)\|_2^2\right]$ where $g$ is terminal function
        \STATE {\bfseries return }$L_{\mathrm{bsde}}, L_{T}$
    \ELSE
        \STATE {\bfseries return }$L_{\mathrm{bsde}}$
    \ENDIF
\end{algorithmic}
}
\end{algorithm*}

\subsection{Variance of the Un-EM-BSDE}
While the independence-based construction removes the discretization-induced bias, debiasing via independent resampling can introduce a variance trade-off. In particular, replacing a biased second-moment objective by a cross-moment computed from independent draws typically involves products of noisy Monte Carlo quantities, which can amplify estimation variance unless the sampling budget is increased. Related bias-variance trade-offs for such cross-sample constructions have been discussed in \citet{hu2025bias}.

Motivated by this variance consideration, we now analyze the variance of the BSDE loss estimators. Let $\widehat{\ell}_{\mathrm{EM}}$, $\widehat{\ell}_{\mathrm{Heun}}$, $\widehat{\ell}_{\mathrm{SG}}^{M}$, $\widehat{\ell}_{\mathrm{SEM}}^{M}$, and $\widehat{\ell}_{\mathrm{UEM}}^{M_1,M_2}$ denote Monte Carlo estimators of $\ell_{\mathrm{EM}}$, $\ell_{\mathrm{Heun}}$, $\ell_{\mathrm{SG}}^{M}$, $\ell_{\mathrm{SEM}}^{M}$, and $\ell_{\mathrm{UEM}}^{M_1,M_2}$, respectively. We summarize the resulting variance comparisons in the following theorem.

\begin{theorem}
\label{variance_comparison}
    Define $\alpha = 2/M - 1/(2M_1) - 1/(2M_2)$ and $\beta = 1/(2M^2) - 1/(4M_1M_2)$. Assume $\beta > 0$ and $\alpha \geq 4/(3M + \beta M^4)$. Then, as $\tau \to 0, \Delta t \to 0$, the estimator variances satisfy
    \begin{equation}
        \mathbb{V}\left[ \hat{\ell}^{M_1, M_2}_{\mathrm{UEM}} \right] \le \mathbb{V}\left[ \hat{\ell}^M_{\mathrm{SG}} \right] = \mathbb{V}\left[ \hat{\ell}^M_{\mathrm{SEM}} \right] \le \mathbb{V}\left[ \hat{\ell}_{\mathrm{EM}} \right] .
    \end{equation}
\end{theorem}

A particularly relevant setting is obtained by setting $M_1=1$ and $M_2=2$. In this case, the condition in Theorem~\ref{variance_comparison} is satisfied with $M=1$, implying that the resulting estimator variance is smaller than that of the standard EM-BSDE loss. Consequently, we can obtain an Un-EM-BSDE objective with improved estimator stability while incurring only a marginal additional sampling cost relative to the original EM-BSDE construction. The proof of Theorem~\ref{variance_comparison} is provided in the Appendix~\ref{proof_variance_comparison}.

\begin{table*}[t]
\centering
    \captionsetup{justification=centering}
    \caption{Mean RL2 errors $(\times10^{-2})$ with standard deviations $(\times10^{-2})$ over three random seeds. The best and second-best results are highlighted in bold with underline and bold, respectively. For AC, the error is computed only at $t=0$.}
\label{RL2_RLT0}
\centering
\small
\begin{tabular}{lcccccc}
\toprule
\multirow{2}{*}{Cases} 
& \multicolumn{3}{c}{Biased methods} &
\multicolumn{3}{c}{Unbiased methods} \\
\cmidrule(lr){2-4}\cmidrule(lr){5-7}
&
EM-BSDE &
Shotgun &
\makecell{Multi-Shot\\EM-BSDE} &

Heun-BSDE &
FS-PINNs &
\makecell{\textbf{Un-EM-BSDE}\\\textbf{(Ours)}} \\
\midrule

\multicolumn{7}{c}{\textbf{Soft constraint}} \\
\midrule
\multirow[c]{1}{*}{HJB}
  & $0.4055$\tiny{$\pm0.0107$} & $1.1409$\tiny{$\pm0.1588$} & $0.1617$\tiny{$\pm0.0144$} & $0.1424$\tiny{$\pm0.0102$} & $\underline{\mathbf{0.0867}}$\tiny{$\pm0.0018$} & $\mathbf{0.1348}$\tiny{$\pm0.0153$} \\
\addlinespace

\multirow[c]{1}{*}{BSB}
  & $0.3483$\tiny{$\pm0.0008$} & $39.9934$\tiny{$\pm23.6762$} & $0.1046$\tiny{$\pm0.0024$} & $0.1030$\tiny{$\pm0.0036$} & $\underline{\mathbf{0.0478}}$\tiny{$\pm0.0015$} & $\mathbf{0.0814}$\tiny{$\pm0.0026$} \\
\addlinespace

\multirow[c]{1}{*}{AC}
  & $0.0462$\tiny{$\pm0.0154$} & $0.0951$\tiny{$\pm0.1270$} & $\mathbf{0.0206}$\tiny{$\pm0.0010$} & $0.0774$\tiny{$\pm0.0027$} & $0.0325$\tiny{$\pm0.0018$} & $\underline{\mathbf{0.0147}}$\tiny{$\pm0.0044$} \\
\addlinespace

\multirow[c]{1}{*}{BZ}
  & $5.3633$\tiny{$\pm0.0338$} & $86.5295$\tiny{$\pm115.6677$} & $\mathbf{5.3379}$\tiny{$\pm0.0304$} & $6.1697$\tiny{$\pm0.0845$} & $174.4345$\tiny{$\pm120.3304$} & $\underline{\mathbf{5.1778}}$\tiny{$\pm0.0533$} \\
\addlinespace

\multirow[c]{1}{*}{PIDE}
  & $1.1094$\tiny{$\pm0.0753$} & $1.2072$\tiny{$\pm0.0344$} & $1.0771$\tiny{$\pm0.0314$} & $\mathbf{0.9199}$\tiny{$\pm0.0448$} & $\underline{\mathbf{0.4973}}$\tiny{$\pm0.0097$} & $1.0878$\tiny{$\pm0.0481$} \\

\midrule
\multicolumn{7}{c}{\textbf{Hard constraint}} \\
\midrule
\multirow[c]{1}{*}{HJB}
  & $0.3433$\tiny{$\pm0.0051$} & $0.2036$\tiny{$\pm0.0110$} & $\underline{\mathbf{0.1172}}$\tiny{$\pm0.0086$} &  $0.2655$\tiny{$\pm0.0876$} & $0.1876$\tiny{$\pm0.0487$} & $\mathbf{0.1444}$\tiny{$\pm0.0275$} \\
\addlinespace

\multirow[c]{1}{*}{BSB}
  & $0.3456$\tiny{$\pm0.0002$} & $0.1629$\tiny{$\pm0.0016$} & $0.0739$\tiny{$\pm0.0002$} & $0.0201$\tiny{$\pm0.0016$} & $\underline{\mathbf{0.0048}}$\tiny{$\pm0.0009$} & $\mathbf{0.0120}$\tiny{$\pm0.0004$} \\
\addlinespace

\multirow[c]{1}{*}{AC}
  & $0.0400$\tiny{$\pm0.0142$} & $0.0136$\tiny{$\pm0.0016$} & $\mathbf{0.0113}$\tiny{$\pm0.0029$} & $0.0912$\tiny{$\pm0.0004$} & $0.0301$\tiny{$\pm0.0004$} & $\underline{\mathbf{0.0034}}$\tiny{$\pm0.0024$} \\
\addlinespace

\multirow[c]{1}{*}{BZ}
  & $1.1912$\tiny{$\pm0.0083$} & $0.2383$\tiny{$\pm0.0403$} & $1.1952$\tiny{$\pm0.0070$} & $1.1885$\tiny{$\pm0.0027$} & $\underline{\mathbf{0.2129}}$\tiny{$\pm0.0063$} & $\mathbf{0.2279}$\tiny{$\pm0.0010$} \\
\addlinespace

\multirow[c]{1}{*}{PIDE}
  & $0.0374$\tiny{$\pm0.0155$} & $0.4057$\tiny{$\pm0.0205$} & $0.0245$\tiny{$\pm0.0087$} & $0.1874$\tiny{$\pm0.0162$} & $\underline{\mathbf{0.0137}}$\tiny{$\pm0.0022$} & $\mathbf{0.0226}$\tiny{$\pm0.0028$} \\

\bottomrule
\end{tabular}

\end{table*}

\section{Experiment} 

In this section, we provide a comparative analysis of the proposed Un-EM-BSDE method against several approaches, including EM-BSDE, FS-PINNs, Shotgun, and the Heun-BSDE. Furthermore, we demonstrate that our debiasing mechanism is easily applicable to other biased methods, such as the Shotgun method. This highlights the applicability of our approach, showing its potential to be extended to a broad range of existing biased numerical frameworks.

We implement our proposed Un-EM-BSDE using Algorithm~\ref{Unbiased_EM-BSDE_algorithm}. Regarding the experimental setup, the hyperparameters are set to $M_1 = 5$ and $M_2 = 5$ based on empirically optimal performance results.
Comprehensive details regarding the setup are provided in Appendix~\ref{ablation}.
For baselines, EM-BSDE follows Algorithm~\ref{EM-BSDE_algorithm}, Heun-BSDE follows Algorithm~\ref{Heun-BSDE_algorithm}, Shotgun follows Algorithm~\ref{Shotgun_algorithm}, and Multi-Shot EM-BSDE follows Algorithm~\ref{Multi-Shot_EM-BSDE_algorithm}. We use $M = 50$ for Shotgun, and $M = 10$ for Multi-Shot EM-BSDE to match the inner-sampling budget of Un-EM-BSDE, i.e., $M_1 + M_2$. All derivatives in the experiments reported in this paper are computed using reverse-mode automatic differentiation.

\subsection{Results on benchmark PDEs}

We evaluate our method on a spectrum of problems—ranging from benchmark PDEs that fit the form of Eq. \eqref{PDE}, including Hamilton–Jacobi–Bellman (HJB), Black-Scholes-Barenblatt (BSB), and Allen-Cahn (AC). By default, we compute the relative $L_{2}$ error (RL2) over the time grid along the trajectory. For AC, since a closed-form trajectory is not available, we instead report the relative error at the prescribed initial state. Full problem statements for HJB, BSB, and AC are provided in Appendix~\ref{Benchmark_Equations}.

\begin{figure}[t]
  \centering
  \includegraphics[width=\columnwidth]{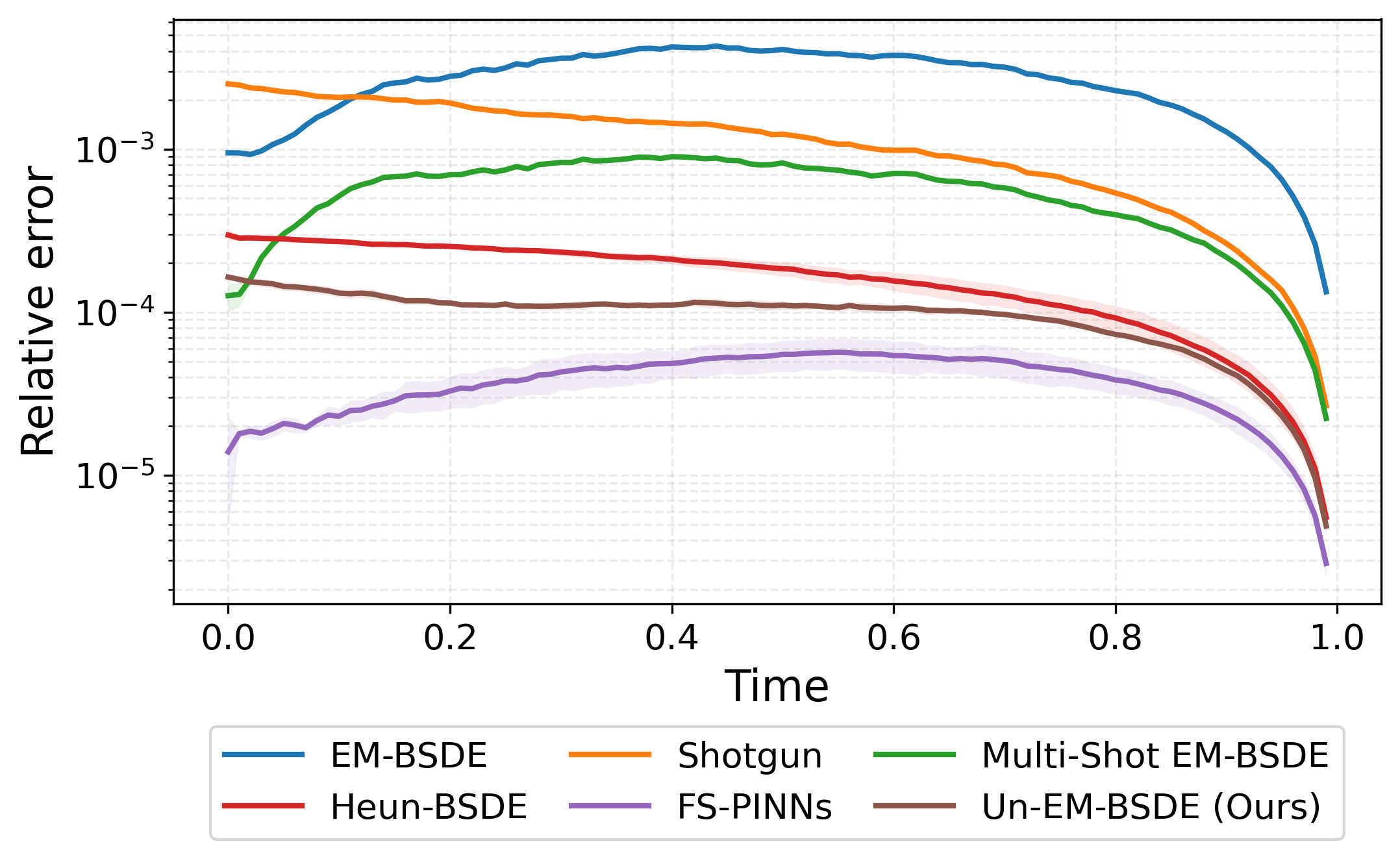}
  \caption{Time-resolved relative error on the BSB benchmark: relative error evaluated at each PDE time $t\in\mathcal{T}$ using the same models as in Table~\ref{RL2_RLT0}. The y-axis is plotted on a logarithmic scale.}
  \label{RL2_over_time}
\end{figure}

\begin{figure}[t]
  \centering
  \includegraphics[width=\columnwidth]{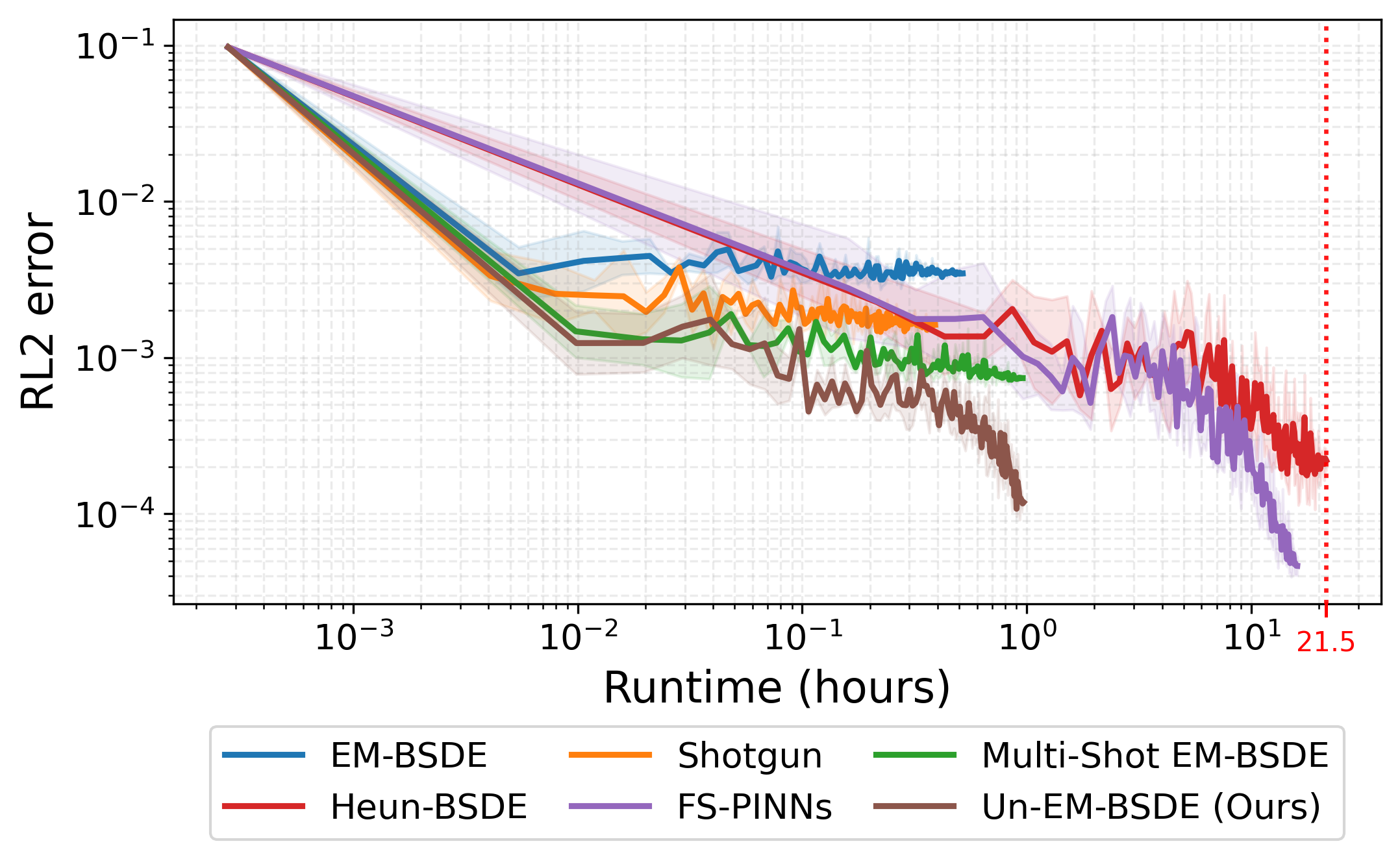}
  \caption{Runtime--accuracy on the BSB benchmark: RL2 recorded over runtime. Both axes are plotted on a log scale.}
  \label{RL2_over_runtime}
\end{figure}

The overall results are summarized in Table~\ref{RL2_RLT0}. Across most benchmark settings, our method achieves consistently strong accuracy and is typically among the top two methods. Figure~\ref{RL2_over_time} further shows a stable reduction in RL2 over time, indicating reliable convergence behavior. We consider both soft-constraint and hard-constraint formulations for enforcing terminal conditions, as both are widely used in practice and offer complementary trade-offs. While hard constraints enforce the conditions by construction, they can hinder optimization when an overly restrictive trial function unintentionally limits the admissible solution set \citep{liu2022a}. Soft constraints, in contrast, preserve flexibility but introduce the classical loss-balancing issue: the relative weighting between residual or self-consistency terms and constraint terms can strongly affect gradient magnitudes and training dynamics unless carefully tuned \citep{wang2022and, bischof2025multi}. Furthermore, Figure~\ref{RL2_over_time} demonstrates that under the hard-constraint setting, our method remains stable and aligns well across the entire time domain.

Across our benchmarks, FS-PINNs train reliably under both formulations and often achieve the best accuracy, reflecting their strong empirical performance when optimization proceeds smoothly. For BSDE-based objectives, however, the choice of constraint enforcement has a more pronounced impact on training dynamics. This is particularly evident under soft constraints, where the balance between self-consistency and constraint penalties directly affects stability and convergence. In this respect, our method is comparatively robust: it performs consistently under both soft and hard constraints and reliably improves upon biased BSDE variants across problems. Under hard constraints, the removal of explicit constraint-loss terms simplifies the objective and reduces tuning overhead, making the benefit of randomized unbiased step-loss estimation more evident.

Additionally, our method demonstrates superior efficiency by avoiding explicit second-order derivative computations.  In high dimensions, the presence of Hessian-related terms can dominate the computational burden and may render other unbiased methods impractical to run at all. As shown in Figure~\ref{RL2_over_runtime}, current unbiased methods, such as FS-PINNs and Heun-BSDE, require over an order of magnitude more runtime. By removing this bottleneck, our approach scales more favorably with dimensionality, providing a clear practical advantage when targeting high-dimensional PDEs. Moreover, as shown in Table~\ref{training_time}, our method maintains a training cost comparable to standard EM-BSDE, while being substantially faster than other unbiased baselines. The resulting accuracy–runtime trade-off for the BSB hard-constraint setting illustrates that our method attains accuracy comparable to competing unbiased approaches at a noticeably lower computational cost.

\subsection{Extensions to complex dynamics}
We move beyond the benchmark PDEs and consider more practically motivated extensions where the standard theoretical assumptions and stability guarantees are less straightforward. Specifically, we study a problem from \citet{bender2008time} that, after the PDE-to-FBSDE transformation, yields a fully-coupled FBSDE because the diffusion coefficient $\sigma$ depends on the unknown solution $u$; and a Partial Integro-Differential Equations (PIDE) example from \citet{ye2024fbsjnn} whose jump-driven, discontinuous dynamics provide a stringent stress test for the proposed training objective. Formal definitions and full problem statements for the fully-coupled and PIDE extensions are provided in Appendix~\ref{BZ_definition} and Appendix~\ref{PIDE_definition}.

From an algorithmic standpoint, these extensions require only lightweight modifications. For the fully-coupled problem, we simply evaluate the network within the rollout since the coefficients depend on the current solution estimate. For the PIDE case, at each time step, we sample the number of jumps and their magnitudes and incorporate the resulting jump contribution in a scheme-agnostic manner. The last two rows for both the soft- and hard-constraint settings in Table~\ref{RL2_RLT0} report the quantitative results for the BZ and PIDE cases. Overall, the soft-constraint formulation exhibits slightly less stable training dynamics in these more challenging settings, whereas the hard-constraint formulation remains stable and delivers competitive accuracy.

\subsection{A generic debiasing wrapper for one-step losses}
\label{debiased_one_step_losses}

Although the debiasing mechanism was introduced in the context of the EM-BSDE loss, it applies more generally to any biased one-step training objective, provided that the one-step loss estimator can be evaluated under the same update rule with different time-step sizes.
As an illustration, we apply this construction to the Shotgun objective, which averages $M$ i.i.d. one-step errors but still carries discretization-induced bias at any fixed $\tau$.

Concretely, applying the same randomized construction to Shotgun yields the following debiased one-step loss
\begin{multline}
    \ell^{M_1, M_2}_{\mathrm{USG}}(\theta, n, x) = \mathbb{E} [ ( \mathrm{Shot}_{M_1}[\mathrm{err}_n^{\mathrm{SG}}(x; u_\theta)]) \\
    \times (\mathrm{Shot}_{M_2}[\mathrm{err}_n^{\mathrm{SG}}(x; u_\theta)]) ] .
\end{multline}

We then evaluate this plug-in variant on the BSB hard-constraint benchmark. Algorithmically, we retain the original Shotgun training pipeline and only replace its biased one-step objective with the debiased estimator above. Following the unbiased EM-BSDE implementation, we further adopt the same group-wise averaging and product aggregation used to form the final one-step loss, and set $M_1 = M_2 = 50$ throughout.

\begin{figure}[t]
  \centering
  \includegraphics[width=\columnwidth]{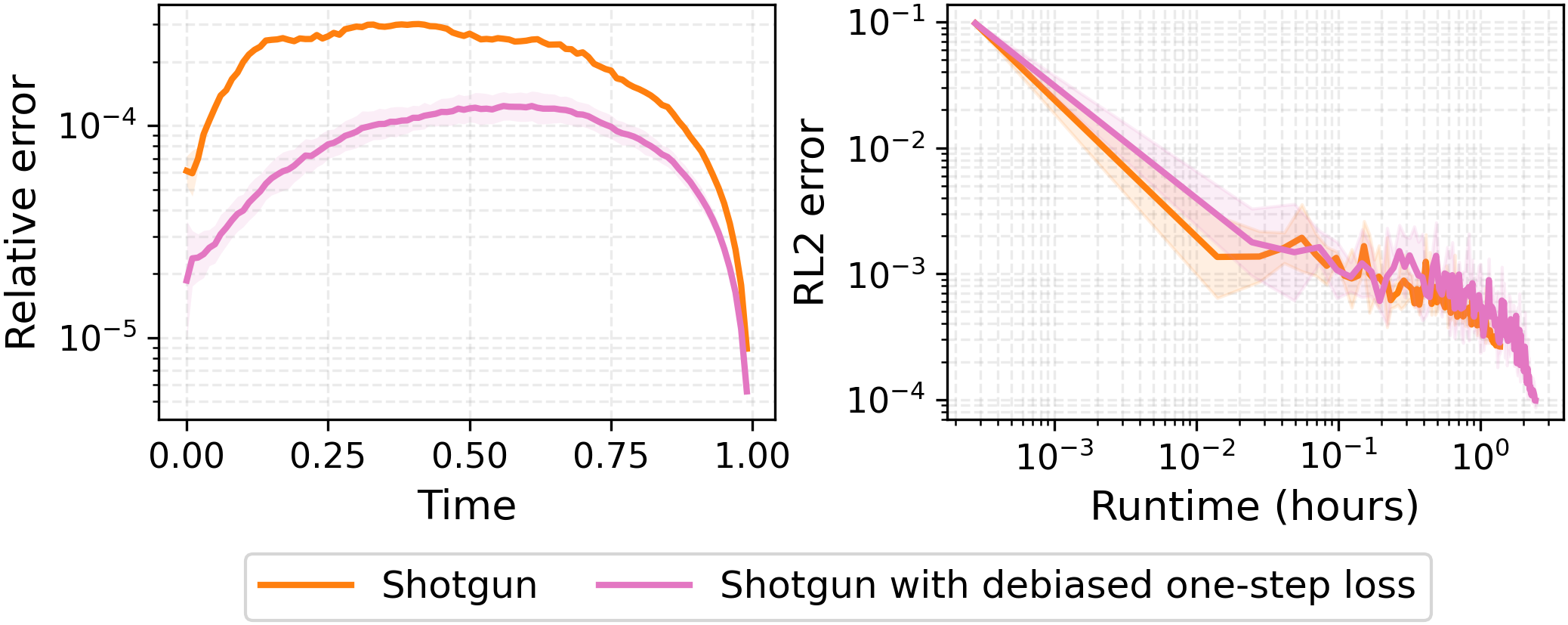}
  \caption{Relative error over PDE time (left) and RL2 over training runtime (right) for Shotgun and its debiased variant.}
  \label{SG_vs_USG}
\end{figure}

The results are visualized in Figure~\ref{SG_vs_USG}. Relative to the corresponding Shotgun run, the wrapper reduces the RL2 error by approximately $2.67\times$, while increasing the training time by $1.78\times$. This suggests that the wrapper can improve Shotgun's accuracy with a moderate computational overhead.

These results suggest that the proposed construction is not tied to EM-BSDE: it can be plugged into other biased one-step losses and yields a clear accuracy gain at a moderate additional cost.

\section{Conclusion}

We presented an unbiased training objective for EM-BSDE–type solvers that removes discretization-induced bias without requiring explicit second-order derivative computations. Across standard benchmarks and more practically motivated extensions, the proposed method achieves consistently strong accuracy while keeping training time comparable to EM-BSDE and substantially lower than Hessian-based unbiased alternatives. Moreover, the same randomized construction can be used as a plug-in correction for other biased one-step objectives, as illustrated by debiasing the Shotgun loss. Future work includes integrating the proposed method with adaptive time-stepping to better allocate computation across time and improve efficiency on stiff or multi-scale dynamics. At the same time, our current theory is established under regularity and boundedness assumptions, and therefore does not yet fully cover more challenging settings such as coupled FBSDEs and PIDEs. Extending the analysis to these broader BSDE-type problems remains an important direction for future work.

\section*{Impact Statement}
This paper presents work whose goal is to advance the field of machine learning for scientific computing, particularly in the numerical solution of high-dimensional partial differential equations. While such methods may contribute to improved modeling and simulation in areas such as finance and the physical sciences, the techniques developed here are general-purpose and do not raise specific ethical concerns beyond those commonly associated with advances in machine learning. The authors do not anticipate any immediate or unique negative societal impacts arising from this work.

\section*{Acknowledgments}

This work was supported by the Institute of Information \& Communications Technology Planning \& Evaluation (IITP) grant funded by the Korea government (MSIT) [RS-2021-II211341, Artificial Intelligence Graduate School Program (Chung-Ang University)]. This work was supported by the National Research Foundation of Korea(NRF) grant funded by the Korea government(MSIT) (RS-2025-02303239 and RS-2026-25497362).


\bibliography{nobias_training}
\bibliographystyle{icml2026}

\newpage
\appendix
\onecolumn

\section{Background}

This section provides additional background that is excluded from the main paper. Specifically, we give additional methodological details and provide algorithmic descriptions of how each method constructs its training loss, along with proofs for the bias of the resulting discretized objectives.

Throughout the paper, we present these algorithms in a batched form, which can be the most efficient option in principle for GPU training.
For runtime reporting, however, we use a full-rollout implementation as the reference protocol, since it is the most commonly used choice in practice and batched execution can lead to substantial GPU-dependent differences in measured wall-clock time.
Importantly, batching and full rollout evaluate the same discretized objective; they differ only in batching and parallelization of the same computations.

\subsection{Deep BSDE methods}
\label{Appendix_DeepBSDE}
We can use Itô's formula to check why the solution of the FBSDE \eqref{FBSDE} comes from the PDE \eqref{PDE} as follows:
\begin{align}
    du(t, X_t)
    &= (\partial _t u(t, X_t)) dt + \nabla u(t, X_t) \cdot dX_t + \frac{1}{2} (dX_t)^T (\nabla^2 u(t, X_t)) (dX_t) \\
    &= (\partial _t u(t, X_t)) dt + [\nabla u \cdot (\mu dt + \sigma dW_t)](t, X_t) + \frac{1}{2} dW_t^T [\sigma^T (\nabla^2 u) \sigma](t, X_t) dW_t \\
    &= \left[ \partial_t u + \mu \cdot \nabla u + \frac{1}{2} \mathrm{Tr}[\sigma^T (\nabla^2 u) \sigma] \right](t, X_t) dt + (\nabla u(t, X_t))^T \sigma(t, X_t) dW_t \\
    &= \mathcal{L}[u](t, X_t) dt + (\nabla u(t, X_t))^T \sigma(t, X_t) dW_t \\
    &= \phi(t, X_t, Y_t, Z_t) dt + Z_t^T \sigma(t, X_t) dW_t .
\end{align}

\begin{algorithm}[tb]
    \caption{EM-BSDE loss algorithm}
    \label{EM-BSDE_algorithm}
{\small
\begin{algorithmic}
    \STATE {\bfseries Input:} batch size $B$, terminal time $T_{\text{end}}$, number of steps $N$, neural network $u_\theta(t, x)$.
    \STATE Construct time grid: Define $\Delta t = T_{\text{end}}/N$, and define $T \in \mathbb{R}^{B \times (N+1)}$ by $T[b, n] = n\Delta t$ for each $b, n$
    \STATE Initialize forward states: Allocate $X \in \mathbb{R}^{B \times (N+1) \times d}$, and set $X[b, 0] = x_0$ for each $b,n$
    \FOR{$n=0$ {\bfseries to} $N-1$}
        \STATE Sample i.i.d. Brownian increments: $\{\Delta W_{n}^{(b)}\}_{b\in[B]} \overset{\text{i.i.d.}}{\sim}\mathcal{N}(0,\Delta t\, I_d)$.
        \STATE Propagate forward states: $X[:, n+1] = X[:, n] + \mu(T[:, n], X[:, n])\Delta t + \sigma(T[:, n], X[:, n])\Delta W_{n}^{(b)}$
    \ENDFOR

    \STATE Evaluate network outputs: Define $Y \in \mathbb{R}^{B\times (N+1)}$ by $Y[b, n] = u_\theta(T[b, n], X[b, n])$ for each $b, n$
    \STATE Evaluate network gradients: Allocate $Z \in \mathbb{R}^{B \times (N+1) \times d}$, and set $Z[b, n] = \nabla_x u_\theta(T[b, n], X[b, n])$ for each $b, n$
    \STATE Compute one-step BSDE prediction: Define $\hat{Y} \in \mathbb{R}^{B \times (N+1)}$ by:
    $$
    \hat{Y}[b, n+1] = Y[b, n] + \phi(T, X, Y, Z)[b, n] \Delta t + [Z^T \sigma(T, X)][b, n]\Delta W_{n}^{(b)} ~~~~ \text{for each }b, n
    $$
    \STATE Compute BSDE self-consistency loss: $L_{\mathrm{bsde}} = \frac{1}{B}\sum_{b=0}^{B-1}\sum_{n=0}^{N-1}\left( Y[b, n+1] - \hat{Y}[b, n+1] \right)^2$
    \STATE Return the BSDE loss (with optional soft-constraint terminal condition loss) as in Algorithm \ref{Unbiased_EM-BSDE_algorithm}
\end{algorithmic}
}
\end{algorithm}

\textbf{EM-BSDE} \citep{raissi2024forward}. We discretize the FBSDE using the Euler–Maruyama scheme and apply the corresponding one-step update.
\begin{equation}
    \begin{cases}
        F_n(x) := x + \mu(t_n, x) \Delta t + \sigma(t_n, x) \Delta W_n , \\
        B_n(x; u) := u(t_n, x) + \phi_u(t_n, x) \Delta t + \nabla u(t_n, x)^T \sigma(t_n, x) \Delta W_n .
    \end{cases}
\end{equation}

Based on the resulting discrete FBSDE, we construct a one-step self-consistency error and the corresponding one-step loss.

\begin{equation}
    \ell_{\mathrm{EM}}(\theta; n, x) := \mathbb{E}\left[|\mathrm{err}_n^{\mathrm{EM}}(x; u_\theta)|^2 \right], ~~~~ \text{where} ~~ \mathrm{err}_n^{\mathrm{EM}}(x; u) := \frac{u(t_{n+1}, F_n(x)) - B_n(x; u)}{\Delta t} .
\end{equation}

The overall training objective is formed by aggregating the per-step losses over the time grid (together with the terminal condition), leading to the implementation summarized in Algorithm~\ref{EM-BSDE_algorithm}. 

\begin{algorithm*}[tb]
    \caption{Heun-BSDE loss algorithm}
    \label{Heun-BSDE_algorithm}
{\small
\begin{algorithmic}
    \STATE {\bfseries Input:} batch size $B$, terminal time $T_{\text{end}}$, number of steps $N$, neural network $u_\theta(t, x)$.
    \STATE Construct time grid and Initialize forward states as in Algorithm \ref{EM-BSDE_algorithm}
    \STATE Allocate intermediate states $\bar{X} \in \mathbb{R}^{B \times (N+1) \times d}$ 
    \FOR{$n=0$ {\bfseries to} $N-1$}
        \STATE Sample i.i.d. Brownian increments: $\{\Delta W_{n}^{(b)}\}_{b\in[B]} \overset{\text{i.i.d.}}{\sim}\mathcal{N}(0,\Delta t\, I_d)$.
        \STATE Compute intermediate states: $\bar{X}[:, n+1] = X[:, n] + \Delta x$ where $\Delta x = \mu^{\text{Heun}}(T, X)[:, n]\Delta t + \sigma(T, X)[:, n]\Delta W_n^{(b)}$
        \STATE Propagate forward states: $X[:, n+1] = X[:, n] + \frac{1}{2}(\Delta x + \bar{\Delta} x)$ where $\bar{\Delta} x = \mu^{\text{Heun}}(T, \bar{X})[:, n+1]\Delta t + \sigma(T, \bar{X})[:, n+1]\Delta W_n^{(b)}$
    \ENDFOR

    Allocate $Y, \bar{Y} \in \mathbb{R}^{B \times (N+1)}$, $Z, \bar{Z} \in \mathbb{R}^{B \times (N+1) \times d}$, and $c, \bar{c} \in \mathbb{R}^{B \times (N+1)}$

    \STATE Evaluate network outputs : $Y[b, n] = u_\theta(T, X)[b, n]$ and $\bar{Y}[b, n] = u_\theta(T, \bar{X})[b, n]$ for each $b, n$
    \STATE Evaluate network gradients : $Z[b, n] = \nabla_x u_\theta(T, X)[b, n]$ and $\bar{Z}[b, n] = \nabla_x u_\theta(T, \bar{X})[b, n]$ for each $b, n$
    \STATE Evaluate weighted network Laplacians : $c[b, n] = \mathrm{Tr}[(\sigma^T (\nabla_x^2 u_\theta) \sigma)(T, X)[b, n]]$ and $\bar{c}[b, n] = \mathrm{Tr}[(\sigma^T (\nabla_x^2 u_\theta) \sigma)(T, \bar{X})[b, n]]$
    \STATE Compute one-step BSDE prediction : Define $\hat{Y} \in \mathbb{R}^{B \times (N+1)}$ by:
    \begin{align*}
        \bar{\hat{Y}}[b, n+1] &= Y[b, n] + \Delta y & \text{where}~\Delta y = \phi_{u_\theta}^{\text{Heun}}(T, X)[b, n] \Delta t + [Z^T \sigma(T, X)][b, n]\Delta W_{n}^{(b)} \\
        \hat{Y}[b, n+1] &= Y[b, n] + \frac{1}{2}(\Delta y + \bar{\Delta} y) & \text{where}~ \bar{\Delta} y = \phi_{u_\theta}^{\text{Heun}}(T, \bar{X})[b, n+1] \Delta t + [\bar{Z}^T \sigma(T, \bar{X})][b, n+1]\Delta W_{n}^{(b)}
    \end{align*}
    \STATE Compute BSDE self-consistency loss: $L_{\mathrm{bsde}} = \frac{1}{B}\sum_{b=0}^{B-1}\sum_{n=0}^{N-1}\left(Y[b, n+1] - \hat{Y}[b, n+1]\right)^2$
    \STATE Return the BSDE loss (with optional soft-constraint terminal condition loss) as in Algorithm \ref{Unbiased_EM-BSDE_algorithm}
\end{algorithmic}
}
\end{algorithm*}

\textbf{Heun-BSDE} \citep{park2025integration}.
While EM-BSDE relies on the Euler-Maruyama discretization of the forward SDE and the corresponding one-step BSDE update, the Heun-BSDE variant is based on a higher-order time integration that is consistent with a Stratonovich formulation. Specifically, we first rewrite the forward dynamics in Stratonovich form 
\begin{equation}
    \begin{cases}
        dX_t = \mu(t, X_t) dt + \sigma(t, X_t) \circ dW_t, \\
        du(t, X_t) = \left[\phi_{u} - \frac{1}{2}\mathrm{Tr}[\sigma^T (\nabla^2 u) \sigma]\right](t, X_t) dt + (\nabla u(t, X_t))^T \sigma(t, X_t) \circ dW_t.
    \end{cases}
\end{equation}
We then discretize it using the stochastic Heun scheme and apply the corresponding one-step BSDE update
\begin{equation}
    \begin{cases}
        \bar{F}_n(x) := x + \mu^{\text{Heun}}(t_n, x) \Delta t + \sigma(t_n, x) \Delta W_n , \\
        \bar{B}_n(x; u) := u(t_n, x) + \phi_u^{\text{Heun}}(t_n, x) \Delta t + \nabla u(t_n, x)^T \sigma(t_n, x) \Delta W_n .
    \end{cases}
\end{equation}
\begin{equation}
    \begin{cases}
        F_n^{\text{Heun}}(x) := x + \frac{\mu^{\text{Heun}}(t_n, x) + \mu^{\text{Heun}}(t_{n+1}, \bar{F}_n(x))}{2}\Delta t + \frac{\sigma(t_n, x) + \sigma(t_{n+1}, \bar{F}_n(x))}{2} \Delta W_n, \\
        B_n^{\text{Heun}}(x; u) := u(t_n, x) + \frac{\phi_u^{\text{Heun}}(t_n, x) + \phi_u^{\text{Heun}}(t_{n+1}, \bar{F}_n(x))}{2} \Delta t + \frac{[(\nabla u)^T \sigma](t_n, x) + [(\nabla u)^T \sigma](t_{n+1}, \bar{F}_n(x))}{2} \Delta W_n .
    \end{cases}
\end{equation}

Here $\mu^{\text{Heun}}$ and $\phi_u^{\text{Heun}}$ incorporate the Stratonovich-Itô corrections associated with moving from Stratonovich dynamics to an Itô representation. These correction terms can be derived rigorously from the Stratonovich–Itô conversion formula via quadratic covariation. From a complementary, discrete-time viewpoint, the same structure is recovered in the Taylor-expansion analysis of the Heun-BSDE loss: the midpoint evaluation produces local contributions proportional to $\Delta W_n \Delta W_n^T$ (See \eqref{forward_correction} and \eqref{backward_correction}), whose expectation equals $I \Delta t$. For completeness, the explicit correction formulas are

\begin{equation}
    \mu^{\text{Heun}} = \mu - \frac{1}{2}\sum_{i=1}^{d}(\partial_{x_i} \sigma) \sigma, ~~~~ \phi_u^{\text{Heun}} = \phi_u - \frac{1}{2} \mathrm{Tr}[\sigma^T (\nabla^2 u) \sigma] - \frac{1}{2} \sum_{i=1}^{d} (\nabla u)^T(\partial_{x_i} \sigma) \sigma .
\end{equation}

Based on the resulting discrete FBSDE, we construct the one-step self-consistency error in the same manner as for EM-BSDE, and define the corresponding one-step loss:

\begin{equation}
    \ell_{\mathrm{Heun}}(\theta; n, x) := \mathbb{E}\left[|\mathrm{err}_n^{\mathrm{Heun}}(x; u_\theta)|^2 \right] , ~~~~ \text{where} ~~ \mathrm{err}_n^{\mathrm{Heun}}(x; u) := \frac{u(t_{n+1}, F_n^{\text{Heun}}(x)) - B_n^{\text{Heun}}(x; u)}{\Delta t} .
\end{equation}

Finally, we aggregate the per-step losses over the time grid and incorporate the terminal condition to form the overall training objective. The resulting Heun-BSDE objective is optimized using the procedure presented in Algorithm~\ref{Heun-BSDE_algorithm}.

\textbf{Shotgun} \citep{xu2025deep}. Compared with the EM-/Heun-BSDE constructions above, the Shotgun method introduces three key modifications. 
First, instead of using a fixed uniform grid shared across the batch, it randomizes the initial time step so that each sample in the batch is trained on a slightly shifted time grid. 
Second, on each such grid, it introduces a finer inner step size $\tau$ and constructs a Brownian-increment-free (i.e., $\Delta W$-free) self-consistency error via antithetic pairing $(\Delta W,-\Delta W)$, so that the resulting error no longer depends explicitly on $\Delta W$.
Third, the per-step error is computed using multiple independent shots (i.e., multiple i.i.d.\ Brownian draws), and the resulting shot-wise errors are averaged to form the final loss, which reduces discretization-induced bias at a fixed time-stepping scheme.

Formally, let $\Delta t := T_{\mathrm{end}}/(N-1)$ and sample the first step as $t_0 = 0$, $t_1 \sim \mathcal{U}(0, \Delta t)$, $t_N = T_{\text{end}}$, and $t_n = t_1 + (n-1) \Delta t$ for $n = 2, ..., N-1$. Denote $\Delta t_n := t_{n+1}-t_n$. We then sample independent Brownian increments $\Delta W_n \sim \mathcal{N}(0, \Delta t_n)$ for each $n$, and generate trajectories using the EM discretization:
\begin{equation}
    \begin{cases}
        \hat{X}_{n+1} = \hat{X}_n + \mu(t_n, \hat{X}_n) \Delta t_n + \sigma(t_n, \hat{X}_n)\Delta W_n, \\
        u(t_{n+1}, \hat{X}_{n+1}) = u(t_n, \hat{X}_n) + \phi_{u}(t_n, \hat{X}_n, u(t_n, \hat{X}_n), \nabla u(t_n, \hat{X}_n)) \Delta t_n + (\nabla u(t_n, \hat{X}_n))^T\sigma(t_n, \hat{X}_n)\Delta W_n.
    \end{cases}
\end{equation}

On each grid, we further introduce a finer time step $\tau$ and sample Brownian increment $\Delta w_n \sim \mathcal{N}(0, \tau I_d)$. Using the antithetic pair $(\Delta w_n, -\Delta w_n)$, we define $F_n^{\pm}(x) := x + \mu(t_n, x) \tau \pm \sigma(t_n, x)\Delta w_n$. We then construct a $\Delta W$-free one-step self-consistency error and the corresponding one-step loss by averaging $M$ independent shots:
\begin{equation}
    \mathrm{err}_n^{\mathrm{SG}}(x; u) = \frac{u(t_n + \tau, F_n^{+}(x)) + u(t_n + \tau, F_n^{-}(x)) - 2u(t_n, x)}{2 \tau} - \phi_{u}(t_n, x),
\end{equation}
\begin{equation}
    \ell^M_{\mathrm{SG}}(\theta; n, x) := \mathbb{E}\left[\left| \mathrm{Shot}_M\left[\mathrm{err}_n^{\mathrm{SG}}(x; u_\theta)\right] \right|^2 \right].
\end{equation}

The overall objective is obtained by aggregating the per-step losses along the time grid, together with the terminal condition; see Algorithm~\ref{Shotgun_algorithm}.

\begin{algorithm*}[tb]
    \caption{Shotgun loss algorithm}
    \label{Shotgun_algorithm}
{\small
\begin{algorithmic}
    \STATE {\bfseries Input:} batch size $B$, terminal time $T_{\text{end}}$, number of steps $N \geq 2$, neural network $u_\theta(t, x)$, finer time step $\tau$, number of shots $M$.
    \STATE Construct time step : Define $\Delta t = T_{\text{end}}/(N-1)$, and sample first time step $\{\Delta t_0^{(b)}\}_{b\in[B]} \overset{\text{i.i.d.}}{\sim}\mathcal{U}(0,\Delta t)$
    \STATE Construct time grid : Define $T \in \mathbb{R}^{B \times (N+1)}$ by $T[b, 0] = 0$, $T[b, 1] = \Delta t_0$, $T[b, N] = T_{\text{end}}$, $T[b, n] = \Delta t_0 + (n-1)\Delta t$ for each $b, n$, and define $\Delta t_n = T[:, n+1] - T[:, n]$ for each $n$ 
    \STATE Initialize shotgun states : Allocate $X^+ \in \mathbb{R}^{B \times N \times M \times d}$ and $X^- \in \mathbb{R}^{B \times N \times M \times d}$
    \STATE Initialize main forward states : Allocate $X_{\text{main}} \in \mathbb{R}^{B \times (N+1) \times d}$, and set $X_{\text{main}}[b, 0] = x_0$ for each $b,n$
    \FOR{$n=0$ {\bfseries to} $N-1$}
        \STATE Sample i.i.d. Brownian increments: $\{\Delta w^{(b)}_{n,i}\}_{b\in[B],\,i\in[M]} \overset{\text{i.i.d.}}{\sim}\mathcal{N}(0, \tau I_d)$ and $\{\Delta W_{n}^{(b)}\}_{b\in[B]} \overset{\text{i.i.d.}}{\sim}\mathcal{N}(0,\Delta t_n I_d)$.
        \STATE Compute shotgun states: $X^{\pm}[:, n, i] = X_{\text{main}}[:, n] + \mu(T, X_{\text{main}})[:, n]\tau \pm \sigma(T, X_{\text{main}})[:, n]\Delta w_{n,i}^{(b)}$ for each $n,i$
        \STATE  Propagate main forward state: $X_{\text{main}}[:, n+1] = X_{\text{main}}[:, n] + \mu(T, X_{\text{main}})[:, n]\Delta t_n + \sigma(T, X_{\text{main}})[:, n]\Delta W_n^{(b)}$
    \ENDFOR
    
    Allocate $Y^+ \in \mathbb{R}^{B \times N \times M}$, $Y^- \in \mathbb{R}^{B \times N \times M}$ $Y_{\text{main}} \in \mathbb{R}^{B \times (N+1)}$, and $Z_{\text{main}} \in \mathbb{R}^{B \times (N+1) \times d}$
     \STATE Evaluate network outputs : $Y^{\pm}[b, n, i] = u_\theta(T[b, n]+\tau, X^{\pm}[b, n, i])$ and $Y_{\text{main}}[b, n] = u_\theta(T, X_{\text{main}})[b, n]$ for each $b, n, i$
    \STATE Evaluate network gradients : $Z_{\text{main}}[b, n] = \nabla_x u_\theta(T, X_{\text{main}})[b, n]$ for each $b, n$
    \STATE Compute drift term : Define $\Phi \in \mathbb{R}^{B \times N \times M}$ by $\Phi[b, n] = \phi(T, X_{\text{main}}, Y_{\text{main}}, Z_{\text{main}})$ for each $b, n$
    \STATE Compute BSDE self-consistency loss
    $$
    L_{\mathrm{bsde}} = \frac{1}{B}\sum_{b=0}^{B-1}\sum_{n=0}^{N-1}\left(\frac{1}{M}\sum_{i=1}^{M}\left( \frac{[Y^+ + Y^- - 2 Y_{\text{main}}][b, n, i]}{2\tau} - \Phi[b, n] \right) \right)^2.
    $$
    \STATE Return the BSDE loss (with optional soft-constraint terminal condition loss) as in Algorithm \ref{Unbiased_EM-BSDE_algorithm}
    
\end{algorithmic}
}
\end{algorithm*}

\textbf{Multi-Shot EM-BSDE.}
The Multi-Shot EM-BSDE method can be viewed as the standard EM-BSDE baseline augmented only with the multiple independent shots mechanism from Shotgun. Equivalently, it matches the proposed Un-EM-BSDE construction with the debiasing step removed, retaining only the shot-averaging operator. Accordingly, the time discretization and the underlying one-step update are identical to those of EM-BSDE, and we continue to use the same EM one-step self-consistency error $\mathrm{err}_n^{\mathrm{EM}}$. Following the same logic as in Shotgun, we evaluate this error using $M$ independent shots and define the corresponding one-step loss via shot averaging:

\begin{equation}
    \ell^M_{\mathrm{SEM}}(\theta; n, x) := \mathbb{E}\left[\left| \mathrm{Shot}_M\left[\mathrm{err}_n^{\mathrm{EM}}(x; u_\theta)\right] \right|^2 \right] .
\end{equation}

We define the overall objective by integrating the per-step losses over the time grid, together with the terminal condition; see Algorithm~\ref{Multi-Shot_EM-BSDE_algorithm}.

\begin{algorithm*}[tb]
    \caption{Multi-Shot EM-BSDE loss algorithm}
    \label{Multi-Shot_EM-BSDE_algorithm}
{\small
\begin{algorithmic}
    \STATE {\bfseries Input:} batch size $B$, terminal time $T_{\text{end}}$, number of steps $N \geq 2$, neural network $u_\theta(t, x)$, finer time step $\tau$, number of shots $M$.
    \STATE Construct time grid: Define $\Delta t = T_{\text{end}}/N$, and define $T \in \mathbb{R}^{B \times (N+1)}$ by $T[b, n] = n\Delta t$ for each $b, n$
    \STATE Initialize candidate states : Allocate $X \in \mathbb{R}^{B \times (N+1) \times M \times d}$, and set $X[b, 0, 0] = x_0$ for each $b$
    \STATE Initialize main forward states : $(t, x) = (T[:, 0], X[:, 0, 0])$
    \FOR{$n=0$ {\bfseries to} $N-1$}
        \STATE Sample i.i.d. Brownian increments: $\{\Delta W_{n,i}^{(b)}\}_{b\in[B],\,i\in[M]} \overset{\text{i.i.d.}}{\sim}\mathcal{N}(0,\Delta t\, I_d)$.
        \STATE Construct candidate next states: $X[:, n+1, i] = x + \mu(t, x)\Delta t + \sigma(t, x)\Delta W_{n,i}^{(b)}$ for each $n,i$
        \STATE Update main forward states : $(t, x) = (T[:, n+1], X[:, n+1, 0])$
    \ENDFOR
    \STATE Construct main forward trajectory: Define $X_{\text{main}} = X[:, :, 0]$
    \STATE Evaluate network outputs : Define $Y \in \mathbb{R}^{B\times (N+1) \times M}$ by $Y[b, n, i] = u_\theta(T[b, n], X[b, n, i])$ for each $b, n, i$
    \STATE Evaluate network gradients : Define $Z_{\text{main}} \in \mathbb{R}^{B \times (N+1) \times d}$ by $Z_{\text{main}}[b, n] = \nabla_x u_\theta(T, X_{\text{main}})[b, n]$ for each $b, n$
    \STATE Compute one-step BSDE prediction : Define $Y_{\text{main}} = Y[:, :, 0]$, and define $\hat{Y} \in \mathbb{R}^{B \times (N+1) \times M}$ by:
    $$
    \hat{Y}[b, n+1, i] = Y_{\text{main}}[b, n] + \phi(T, X_{\text{main}}, Y_{\text{main}}, Z_{\text{main}})[b, n] \Delta t + [Z_{\text{main}}^T \sigma(T, X_{\text{main}})][b, n]\Delta W_{n,i}^{(b)} ~~~~ \text{for each }b, n, i
    $$
    \STATE Compute BSDE self-consistency loss: $L_{\mathrm{bsde}} = \frac{1}{B}\sum_{b=0}^{B-1}\sum_{n=0}^{N-1}\left( \frac{1}{M}\sum_{i=0}^{M-1}(Y[b, n+1, i] - \hat{Y}[b, n+1, i]) \right)^{2}$
    \STATE Return the BSDE loss (with optional soft-constraint terminal condition loss) as in Algorithm \ref{Unbiased_EM-BSDE_algorithm}
\end{algorithmic}
}
\end{algorithm*}

\subsection{Proof of bias for the existing BSDE loss}

\label{proof_existing_BSDE_bias}

In this section, we restate and prove Lemma~\ref{bias_EM}, Lemma~\ref{bias_Heun}, Proposition~\ref{bias_shotgun}, and Proposition~\ref{bias_multishot_EM}.

\begin{restatelemma}{bias_EM}{Bias of the EM-BSDE loss \citep{park2025integration}}
Suppose that $\mu, \sigma$ are bounded and $u_\theta$ is $C^{1,2}$. Then we have
\begin{equation}
    \ell_{\mathrm{EM}}(\theta; n, x) = \left([\mathcal{L}[u_\theta] - \phi_{u_\theta}](t_n, x)\right)^2 + \frac{1}{2}\mathrm{Tr}[(\sigma^T(\nabla^2 u_\theta) \sigma)^2](t_n, x) + O(\Delta t^{1/2}) .
\end{equation}
\end{restatelemma}

\begin{proof}
    By Taylor expansion, for any $n \in [N]$ and $x \in \mathbb{R}^d$, we obtain:
    \begin{align}
        &u_\theta(t_n + \Delta t, F_n(x)) - u_\theta(t_n, x) \\
        \approx~& (\partial_t u_\theta(t_n, x)) \Delta t + \nabla u_\theta(t_n, x) \cdot (F_n(x) - x) + \frac{1}{2}(F_n(x) - x)^T(\nabla^2 u_\theta(t_n, x))(F_n(x) - x) \label{EM_approx_1} \\
        \approx~& (\partial_t u_\theta(t_n, x))\Delta t + [\nabla u_\theta \cdot (\mu \Delta t + \sigma \Delta W_n)](t_n, x) + \frac{1}{2}\Delta W_n^T [\sigma^T (\nabla^2 u_\theta) \sigma](t_n, x) \Delta W_n \label{EM_approx_2} \\
        \approx~& (\partial_t u_\theta(t_n, x))\Delta t + [\nabla u_\theta \cdot (\mu \Delta t + \sigma \Delta W_n)](t_n, x) + \frac{1}{2}\mathrm{Tr}[\sigma^T (\nabla^2 u_\theta) \sigma](t_n, x)\Delta t \label{EM_approx_3} \\
        =~& \left[ \partial_t u_\theta + \mu \cdot \nabla u_\theta + \frac{1}{2}\mathrm{Tr}[\sigma^T (\nabla^2 u_\theta) \sigma] \right](t_n, x) \Delta t + (\nabla u_\theta(t_n, x))^T \sigma(t_n, x) \Delta W_n \\
        =~& \mathcal{L}[u_\theta](t_n, x)\Delta t + (\nabla u_\theta(t_n, x))^T\sigma(t_n, x)\Delta W_n.
    \end{align}
    Thus, we can write
    \begin{equation}
        \label{taylor-EM}
        u_\theta(t_n + \Delta t, F_n(x)) - u_\theta(t_n, x) = \mathcal{L}[u_\theta](t_n, x) \Delta t + (\nabla u_\theta (t_n, x))^T \sigma(t_n, x) \Delta W_n + \epsilon(n, x; u_\theta) .
    \end{equation}
    where
    \begin{align}
        \epsilon(n, x; u_\theta)
        =~& \frac{1}{2}\Delta W_n^T [\sigma^T(\nabla^2 u_\theta) \sigma](t_n, x) \Delta W_n - \frac{1}{2}\mathrm{Tr}[\sigma^T(\nabla^2 u_\theta)\sigma](t_n, x)\Delta t \label{EM_approx_err_1} & [~\because \eqref{EM_approx_3}~] \\
        &+ \frac{1}{2} [\mu^T (\nabla^2 u_\theta) \mu](t_n, x)\Delta t^2 - \Delta W_n^T [\sigma^T (\nabla^2 u_\theta) \mu](t_n, x) \Delta t \label{EM_approx_err_2} & [~\because \eqref{EM_approx_2}~] \\
        &+ \frac{1}{2}(\partial_{tt} u_\theta(t_n, x)) \Delta t^2 + (\text{third- and higher-order terms}) \label{EM_approx_err_3} & [~\because \eqref{EM_approx_1}~] .
    \end{align}

Denote $H_{u_\theta} = \sigma^T(\nabla^2 u_\theta)\sigma$, and let $\xi_n = (\Delta t)^{-1/2}\Delta W_n$. We have
\begin{align}
    \mathrm{err}_n^{\mathrm{EM}}(x; u_\theta) &= \frac{u_\theta(t_n + \Delta t, F_n(x)) - B_n(x; u_\theta)}{\Delta t} \\
    &= \frac{(\mathcal{L}[u_\theta](t_n, x) - \phi(t_n, x; u_\theta))\Delta t + \epsilon(n, x; u_\theta)}{\Delta t} \\
    &= [\mathcal{L}[u_\theta] - \phi_{u_\theta}](t_n, x) + \frac{1}{2}[\xi_n^T H_{u_\theta} \xi_n - \mathrm{Tr}[H_{u_\theta}]](t_n, x) + O(\Delta t^{1/2}).
    \label{EM-error}
\end{align}

Therefore we have
\begin{align}
    \ell_{\mathrm{EM}}(\theta; n, x) &= \mathbb{E}\left[|\mathrm{err}_n^{\mathrm{EM}}(x; u_\theta)|^2 \right] \\
    &= \mathbb{E}\left[ \left|[\mathcal{L}[u_\theta] - \phi_{u_\theta}](t_n, x) + \frac{1}{2}[\xi_n^T H_{u_\theta} \xi_n - \mathrm{Tr}[H_{u_\theta}]](t_n, x)\right|^2 \right] + O(\Delta t^{1/2}) \\
    &= \left([\mathcal{L}[u_\theta] - \phi_{u_\theta}](t_n, x)\right)^2 + \frac{1}{4}\mathbb{E}\left[ \left| [\xi_n^T H_{u_\theta} \xi_n - \mathrm{Tr}[H_{u_\theta}]](t_n, x)\right|^2 \right] + O(\Delta t^{1/2}) \\
    &= \left([\mathcal{L}[u_\theta] - \phi_{u_\theta}](t_n, x)\right)^2 + \frac{1}{2} \mathrm{Tr}[H_{u_\theta}^2](t_n, x)+ O(\Delta t^{1/2}).
\end{align}
\end{proof}

\begin{remark}
The bound in Lemma~\ref{bias_EM} is not tight in the strict sense. Indeed, since the Brownian increment satisfies $\mathbb{E}[\Delta W_n] = 0$, the $\Delta t^{3/2}$-order contributions appearing in the term in \eqref{EM_approx_err_2} and in the third-order remainder in \eqref{EM_approx_err_3} vanish after taking expectation once multiplied with the corresponding deterministic factors. Consequently, under an exact expectation, the remainder can be sharpened from $O(\Delta t^{1/2})$ to $O(\Delta t)$. \\
However, in practice, we do not have access to the exact conditional expectation when estimating the one-step loss, and the above cancellation cannot be exploited. As a result, $\Delta W_n$ should be viewed as a random variable of scale $\Delta t^{1/2}$ (i.e., $\Delta W_n = \Delta t^{1/2} \xi$ with $\xi \sim \mathcal{N}(0, I)$), and the corresponding $O(\Delta t^{1/2})$ control represents the effective bound relevant for stochastic estimation.
\end{remark}

\begin{restatelemma}{bias_Heun}{Unbiasedness of the Heun-BSDE loss \citep{park2025integration}}
Suppose that $\mu, \sigma$, and $\phi_{u_\theta}$ are all in $C^{1,1}$, and $u_\theta$ is in $C^{1,3}$. Then we have:
\begin{equation}
    \ell_{\mathrm{Heun}}(\theta; n, x) = \left([\mathcal{L}[u_\theta] - \phi_{u_\theta}](t_n, x)\right)^2 + O(\Delta t^{1/2}).
\end{equation}
\end{restatelemma}

\begin{proof}
By Taylor expansion, for any $n \in [N]$ and $x \in \mathbb{R}^d$, we obtain
\begin{align}
    \mu^{\text{Heun}}(t_n + \Delta t, \bar{F}_n(x)) &= \mu^{\text{Heun}}(t_n, x) + O(\Delta t^{1/2}), \\
    \sigma(t_n + \Delta t, \bar{F}_n(x)) &= \sigma(t_n, x) + \sum_{i=1}^{d} \partial_{x_i} \sigma(t_n, x)(\sigma(t_n, x)\Delta W_n)_i + O(\Delta t), \\
    \phi_{u_\theta}^{\text{Heun}}(t_n + \Delta t, \bar{F}_n(x)) &= \phi_{u_\theta}^{\text{Heun}}(t_n, x) + O(\Delta t^{1/2}), \\
    [(\nabla u_\theta)^T \sigma](t_n+\Delta t, \bar{F}_n(x)) &= [(\nabla u_\theta)^T \sigma](t_n, x) + \sum_{i=1}^{d} \partial_{x_i} [(\nabla u_\theta)^T \sigma](t_n, x) (\sigma(t_n, x) \Delta W_n)_i + O(\Delta t).
\end{align}

Denote $H_{u_\theta} = \sigma^T(\nabla^2 u_\theta)\sigma$ and let $\xi_n = (\Delta t)^{-1/2}\Delta W_n$. Then we have
\begin{align}
    &F^{\text{Heun}}_n(x) - x \\
    =~& \frac{\mu^{\text{Heun}}(t_n, x) + \mu^{\text{Heun}}(t_n + \Delta t, \bar{F}_n(x))}{2}\Delta t + \frac{\sigma(t_n, x) + \sigma(t_n + \Delta t, \bar{F}_n(x))}{2}\Delta W_n \\
    =~& \left( \mu^{\text{Heun}}(t_n, x) + \frac{1}{2} \sum_{i=1}^{d} \partial_{x_i}\sigma(t_n, x)(\sigma(t_n, x)\xi_n)_i \xi_n \right) \Delta t + \sigma(t_n, x) \Delta W_n + O(\Delta t^{3/2}) \label{forward_correction}.
\end{align}
\begin{align}
    &B^{\text{Heun}}_n(x; u_\theta) - u_\theta(t_n, x) \\
    =~& \frac{\phi_{u_\theta}^{\text{Heun}}(t_n, x) + \phi_{u_\theta}^{\text{Heun}}(t_n + \Delta t, \bar{F}_n(x))}{2}\Delta t + \frac{[(\nabla u_\theta)^T \sigma](t_n, x) + [(\nabla u_\theta)^T \sigma](t_n + \Delta t, \bar{F}_n(x))}{2}\Delta W_n \\
    =~& \left( \phi_{u_\theta}^{\text{Heun}}(t_n, x) + \frac{1}{2} \sum_{i=1}^{d} \partial_{x_i}[(\nabla u_\theta)^T \sigma](t_n, x)(\sigma(t_n, x) \xi_n)_i \xi_n \right)\Delta t \\
    &+ [(\nabla u_\theta)^T \sigma](t_n, x) \Delta W_n + O(\Delta t^{3/2}) \\
    =~& \left( \phi_{u_\theta}^{\text{Heun}}(t_n, x) + \frac{1}{2} \sum_{i=1}^{d}[\partial_{x_i}(\nabla u_\theta(t_n, x))^T \sigma(t_n, x) + (\nabla u_\theta(t_n, x))^T\partial_{x_i}\sigma(t_n, x)](\sigma(t_n, x) \xi_n)_i \xi_n \right) \Delta t \\
    &+ [(\nabla u_\theta)^T \sigma](t_n, x) \Delta W_n + O(\Delta t^{3/2}) \\ 
    =~& \left[ \phi_{u_\theta}^{\text{Heun}} + \frac{1}{2} \xi_n^T H_{u_\theta} \xi_n + \frac{1}{2} \sum_{i=1}^{d}(\nabla u_\theta)^T \partial_{x_i} \sigma (\sigma \xi_n)_i \xi_n \right](t_n, x) \Delta t \\ \label{backward_correction}
    &+ [(\nabla u_\theta)^T \sigma](t_n, x) \Delta W_n + O(\Delta t^{3/2}) .
\end{align}

And by Taylor expansion, we have
\begin{align}
    &u_\theta(t_n + \Delta t, F_n^{\text{Heun}}(x)) - u_\theta(t_n, x) \\
    =~& (\partial_t u_\theta(t_n, x))\Delta t + \nabla u_\theta \cdot (F_n^{\text{Heun}}(x) - x) + \frac{1}{2}(F_n^{\text{Heun}}(x) - x)^T\nabla^2 u_\theta(t_n, x) (F_n^{\text{Heun}}(x) - x) + O(\Delta t^{3/2}) \\
    =~& \left[ \partial_t u_\theta + \mu^{\text{Heun}} \cdot \nabla u_\theta + \frac{1}{2} \xi_n^T H_{u_\theta} \xi_n + \frac{1}{2} \sum_{i=1}^{d} (\nabla u_\theta)^T \partial_{x_i}\sigma(\sigma \xi_n)_i \xi_n \right](t_n, x) \Delta t \\
    &+ [(\nabla u_\theta)^T \sigma](t_n, x) \Delta W_n + O(\Delta t^{3/2}).
\end{align}

Therefore, we have
\begin{align}
    \mathrm{err}_n^{\mathrm{Heun}}(x; u_\theta)
    &= \frac{u_\theta(t_n + \Delta t, F_n^{\text{Heun}}(x)) - B_n^{\text{Heun}}(x; u_\theta)}{\Delta t} \\
    &= \frac{\left(u_\theta(t_n + \Delta t, F_n^{\text{Heun}}(x)) - u_\theta(t_n, x)\right) - \left(B_n^{\text{Heun}}(x; u_\theta) - u_\theta(t_n, x) \right)}{\Delta t} \\
    &= \left[ \partial_t u_\theta + \mu^{\text{Heun}} \cdot \nabla u_\theta - \phi_{u_\theta}^{\text{Heun}} \right] (t_n, x) + O(\Delta t^{1/2}) \\
    &= \left[ \partial_t u_\theta + \mu \cdot \nabla u_\theta + \frac{1}{2} \mathrm{Tr}[H_{u_\theta}] - \phi_{u_\theta}  \right] (t_n, x) + O(\Delta t^{1/2}) \\
    &= [\mathcal{L}[u_\theta] - \phi_{u_\theta}](t_n, x) + O(\Delta t^{1/2}),
\end{align}
\begin{equation}
    \therefore~~ \ell_{\text{Heun}}(\theta; n, x) = \mathbb{E}\left[\left| \mathrm{err}_n^{\text{Heun}}(x; u_\theta)\right|^2\right] = \left([\mathcal{L}[u_\theta] - \phi_{u_\theta}](t_n, x) \right)^2 + O(\Delta t^{1/2}) .
\end{equation}
\end{proof}

\begin{restateproposition}{bias_shotgun}{Bias of the Shotgun loss \citep{xu2025deep}}
Suppose that $\mu, \sigma$ are bounded and $u_\theta$ is $C^{1,2}$. Let $\tau>0$ denote the time increment used in the Shotgun construction. Then we have:
\begin{equation}
    \ell^M_{\mathrm{SG}}(\theta; n, x) = \left([\mathcal{L}[u_\theta] - \phi_{u_\theta}](t_n, x)\right)^2 + \frac{1}{2M}\mathrm{Tr}[(\sigma^T(\nabla^2 u_\theta) \sigma)^2](t_n, x) + O(\tau).
\end{equation}
\end{restateproposition}

\begin{proof}
Define $x^{\pm} := x + \mu(t_n, x) \tau \pm \sigma (t_n, x) \Delta w_n$. Then, using the same derivation as in \eqref{taylor-EM}, we obtain
\begin{equation}
    u_\theta(t_n + \tau, x^{\pm}) - u_\theta(t_n, x) = \mathcal{L}[u_\theta](t_n, x)\tau \pm (\nabla u_\theta(t_n, x))^T \sigma(t_n, x) \Delta w_n + \epsilon^{\pm}(n, x; u_\theta),
\end{equation}
where
\begin{align}
    \epsilon^{\pm}(n, x; u_\theta)
    =~& \frac{1}{2}\Delta w_n^T [\sigma^T(\nabla^2 u_\theta) \sigma](t_n, x) \Delta w_n - \frac{1}{2}\mathrm{Tr}[\sigma^T(\nabla^2 u_\theta)\sigma](t_n, x)\tau  \\
    &+ \frac{1}{2} [\mu^T (\nabla^2 u_\theta) \mu](t_n, x)\tau^2 \mp \Delta w_n^T [\sigma^T (\nabla^2 u_\theta) \mu](t_n, x) \tau \\
    &+ \frac{1}{2}(\partial_{tt} u_\theta(t_n, x)) \tau^2 + (\text{third- and higher-order terms}) .
\end{align}

Denote $H_{u_\theta} = \sigma^T (\nabla^2 u_\theta) \sigma$, and let $\xi_n = \tau^{-1/2}\Delta w_n$. We have
\begin{align}
    \mathrm{err}_n^{\mathrm{SG}}(x; u_\theta) &= \frac{u_\theta(t_n + \tau, x^+) + u_\theta(t_n + \tau, x^-) - 2u_\theta(t_n, x)}{2 \tau} - \phi_{u_\theta}(t_n, x) \\
    &= [\mathcal{L}[u_\theta] - \phi_{u_\theta}](t_n, x) + \frac{\epsilon^+(n, x; u_\theta) + \epsilon^-(n, x; u_\theta)}{2\tau} \label{shotgun_approx} \\
    &= [\mathcal{L}[u_\theta] - \phi_{u_\theta}](t_n, x) + \frac{1}{2}[\xi_n^T H_{u_\theta} \xi_n - \mathrm{Tr}[H_{u_\theta}]](t_n, x) + O(\tau) .
\end{align}

Given $M$ independent samples $\{\Delta w_{n,i}\}_{i=1}^M \overset{\mathrm{i.i.d.}}{\sim} \mathcal{N}(0,\tau I_d)$, let $\xi_{n,i} = \tau^{-1/2}\Delta w_{n,i}$. Then the one-step loss is formed as follows:

\begin{align}
    \ell^M_{\mathrm{SG}}(\theta; n, x) &= \mathbb{E}\left[\left| \mathrm{Shot}_M\left[\mathrm{err}_n^{\mathrm{SG}}(x; u_\theta)\right] \right|^2 \right] \\
    &= \mathbb{E}\left[ \left|\frac{1}{M}\sum_{i=1}^{M}\left([\mathcal{L}[u_\theta] - \phi_{u_\theta}](t_n, x) + \frac{1}{2}[\xi_{n,i}^T H_{u_\theta} \xi_{n,i} - \mathrm{Tr}[H_{u_\theta}]](t_n, x)\right) \right|^2 \right] + O(\tau) \\
    &= \mathbb{E}\left[ \left|[\mathcal{L}[u_\theta] - \phi_{u_\theta}](t_n, x) + \frac{1}{2M}\sum_{i=1}^{M}[\xi_{n,i}^T H_{u_\theta} \xi_{n,i} - \mathrm{Tr}[H_{u_\theta}]](t_n, x) \right|^2 \right] + O(\tau) \\
    &= ([\mathcal{L}[u_\theta] - \phi_{u_\theta}](t_n, x))^2 + \frac{1}{4M^2}\sum_{i=1}^{M}\mathbb{E}\left[ \left| [\xi_{n, i}^T H_{u_\theta} \xi_{n, i} - \mathrm{Tr}[H_{u_\theta}]](t_n, x)\right|^2 \right] + O(\tau) \\
    &= ([\mathcal{L}[u_\theta] - \phi_{u_\theta}](t_n, x))^2 + \frac{1}{2M}\mathrm{Tr}[H_{u_\theta}^2](t_n, x) + O(\tau).
\end{align}

\end{proof}

\begin{remark}
The $O(\tau)$ remainder in Proposition~\ref{bias_shotgun} reflects a stronger cancellation mechanism than the one discussed around \eqref{EM_approx_err_2}--\eqref{EM_approx_err_3}. 
There, the $\Delta t^{3/2}$-order contributions are proportional to the Brownian increment and disappear only after taking expectation, so the cancellation is not directly available at the level of a single Monte Carlo realization.
In contrast, the Shotgun approximation in \eqref{shotgun_approx} combines the two branch errors through $\varepsilon^{+}+\varepsilon^{-}$, where $\varepsilon^{\pm}$ correspond to the antithetic pair $(\Delta w_n,-\Delta w_n)$ at step size $\tau$. 
This pairing enforces a pathwise sign cancellation: all terms that are odd in $\Delta w_n$ cancel in $\varepsilon^{+}+\varepsilon^{-}$, which removes the leading $\tau^{3/2}$-order contributions directly at the discrete level. 
Consequently, the effective remainder improves to $O(\tau)$ without relying on access to the exact expectation.
\end{remark}

\begin{restateproposition}{bias_multishot_EM}{Bias of the Multi-Shot EM loss}
Suppose that $\mu, \sigma$ are bounded and $u_\theta$ is $C^{1,2}$. Then we have:
\begin{equation}
    \ell^M_{\mathrm{SEM}}(\theta; n, x) = \left([\mathcal{L}[u_\theta] - \phi_{u_\theta}](t_n, x)\right)^2 + \frac{1}{2M}\mathrm{Tr}[(\sigma^T(\nabla^2 u_\theta) \sigma)^2](t_n, x) + O(\Delta t^{1/2}).
\end{equation}
\end{restateproposition}

\begin{proof}
Since the Multi-Shot EM construction is obtained by applying $\mathrm{Shot}_M[\cdot]$ to the standard EM one-step error, the bias expansion is a direct consequence of the EM error approximation in \eqref{EM-error}, combined with the shot-averaging argument used for Shotgun. Given $M$ i.i.d.\ samples $\{\Delta w_{n,i}\}_{i=1}^M \sim \mathcal{N}(0,\Delta t I_d)$, denote $H_{u_\theta} = \sigma^T (\nabla^2 u_\theta) \sigma$, and let $\xi_{n,i} = \Delta t^{-1/2}\Delta w_{n,i}$. Then we have:

\begin{align}
    \ell^M_{\mathrm{SEM}}(\theta; n, x) &= \mathbb{E}\left[\left| \mathrm{Shot}_M\left[\mathrm{err}_n^{\mathrm{EM}}(x; u_\theta)\right] \right|^2 \right] \\
    &= \mathbb{E}\left[ \left|\frac{1}{M}\sum_{i=1}^{M}\left([\mathcal{L}[u_\theta] - \phi_{u_\theta}](t_n, x) + \frac{1}{2}[\xi_{n,i}^T H_{u_\theta} \xi_{n,i} - \mathrm{Tr}[H_{u_\theta}]](t_n, x)\right) \right|^2 \right] + O(\Delta t^{1/2}) \\
    &= ([\mathcal{L}[u_\theta] - \phi_{u_\theta}](t_n, x))^2 + \frac{1}{2M}\mathrm{Tr}[H_{u_\theta}^2](t_n, x) + O(\Delta t^{1/2}).
\end{align}

\end{proof}

\section{Proofs of Lemma~\ref{bias_unbiased_EM} and Theorem~\ref{consistency_unbiased_EM}}
\label{proof_Un-EM-BSDE}

In this section, we restate and prove Lemma~\ref{bias_unbiased_EM} and Theorem~\ref{consistency_unbiased_EM}.

\begin{restatelemma}{bias_unbiased_EM}{Unbiasedness of the Un-EM-BSDE loss}
Suppose that $\mu, \sigma$ are bounded and $u_\theta$ is $C^{1,2}$. Then we have
\begin{equation}
    \ell^{M_1, M_2}_{\mathrm{UEM}}(\theta; n, x) = \left([\mathcal{L}[u_\theta] - \phi_{u_\theta}](t_n, x)\right)^2 + O(\Delta t^{1/2}) .
\end{equation}
\end{restatelemma}

\begin{proof}
Recall the EM one-step error $\mathrm{err}^{\mathrm{EM}}_n(x;u_\theta)$ characterized in \eqref{EM-error}.  Given $M_1 + M_2$ independent samples $\{\Delta w_{n,i}\}_{i=1}^{M_1 + M_2} \overset{\mathrm{i.i.d.}}{\sim} \mathcal{N}(0,\Delta t\, I_d)$, let $\xi_{n, i} = \Delta t^{-1/2} \Delta w_{n,i}$. We apply the shot-averaging operator to $\mathrm{err}^{\mathrm{EM}}_n(x;u_\theta)$ and examine its expectation:

\begin{align}
    &\mathbb{E}\left[ \mathrm{Shot}_{M_1}\left[\mathrm{err}_n^{\mathrm{EM}}(x; u_\theta)\right] \right] \\
    =~& \frac{1}{M_1}\sum_{i=1}^{M_1} \mathbb{E}\left[ [\mathcal{L}[u_\theta] - \phi_{u_\theta}](t_n, x) + \frac{1}{2}[\xi_{n,i}^T H_{u_\theta} \xi_{n,i} - \mathrm{Tr}[H_{u_\theta}]](t_n, x) \right] + O(\Delta t^{1/2}) \\
    =~& [\mathcal{L}[u_\theta] - \phi_{u_\theta}](t_n, x) + O(\Delta t^{1/2}),
\end{align}

\begin{align}
    &\mathbb{E}\left[ \mathrm{Shot}_{M_2}\left[\mathrm{err}_n^{\mathrm{EM}}(x; u_\theta)\right] \right] \\
    =~& \frac{1}{M_2}\sum_{i=M_1 + 1}^{M_1 + M_2} \mathbb{E}\left[ [\mathcal{L}[u_\theta] - \phi_{u_\theta}](t_n, x) + \frac{1}{2}[\xi_{n,i}^T H_{u_\theta} \xi_{n,i} - \mathrm{Tr}[H_{u_\theta}]](t_n, x) \right] + O(\Delta t^{1/2}) \\
    =~& [\mathcal{L}[u_\theta] - \phi_{u_\theta}](t_n, x) + O(\Delta t^{1/2}).
\end{align}

Using these two independent shot-averaged estimates, we can write the Un-EM-BSDE one-step loss as:

\begin{align}
    \ell^{M_1, M_2}_{\mathrm{UEM}}(\theta, n, x) &= \mathbb{E}\left[ ( \mathrm{Shot}_{M_1}[\mathrm{err}_n^{\mathrm{EM}}(x; u_\theta)]) (\mathrm{Shot}_{M_2}[\mathrm{err}_n^{\mathrm{EM}}(x; u_\theta)])\right] \\
    &= \mathbb{E}\left[ \mathrm{Shot}_{M_1}[\mathrm{err}_n^{\text{EM}}(x; u_\theta)] \right] \mathbb{E}\left[ \mathrm{Shot}_{M_2}[\mathrm{err}_n^{\text{EM}}(x; u_\theta)] \right] \\
    &= \left( [\mathcal{L}[u_\theta] - \phi_{u_\theta}](t_n, x)\right)^2 + O(\Delta t^{1/2}).
\end{align}

\end{proof}

\begin{restatetheorem}{consistency_unbiased_EM}{Consistency of the Un-EM-BSDE loss}
Suppose that $\mu, \sigma, \phi_{u_\theta} \in C^{1,0}$, $u_\theta \in C^{1,2}$, and $\Delta t \le 1$. We have
\begin{equation}
    \frac{1}{N}\sum_{n=0}^{N-1} \mathbb{E}_{\hat{X}_n}[\ell_{\mathrm{UEM}}(\theta; n, \hat{X}_n)]
    = \frac{1}{|\mathcal{T}|} \int_\mathcal{T} \mathbb{E}\left[ \left( [\mathcal{L}[u_\theta] - \phi_{u_\theta}](t, X_t) \right)^2 \right] dt \\
    + O(\Delta t^{1/2}).
\end{equation}
\end{restatetheorem}

\begin{proof}
The argument is identical to that of \citet[Theorem~4.2]{park2025integration}, except that Lemma~\ref{bias_EM} is substituted by Lemma~\ref{bias_unbiased_EM}. Concretely, by Lemma~\ref{bias_unbiased_EM}, we obtain
\begin{equation}
    \frac{1}{N} \sum_{n=0}^{N-1}\mathbb{E}_{\hat{X}_n}[\ell_{\mathrm{UEM}}(\theta; n, \hat{X}_n)] = \frac{1}{N}\sum_{n=0}^{N-1} \mathbb{E}_{\hat{X}_n}\left[ ([\mathcal{L}[u_\theta] - \phi_{u_\theta}](t_n, \hat{X}_n))^2 \right] + O(\Delta t^{1/2}).
\end{equation}
Under the regularity assumptions $f,g,\phi_{u_\theta}\in C^{0,1}$ and $u_\theta\in C^{2,1}$,
the map $(t,x)\mapsto ([\mathcal{L} - \phi_{u_\theta}](t,x))^2$ is Lipschitz continuous on $\mathcal{T}\times\mathbb{R}^d$. Moreover, the forward process $(X_t)_{t\in\mathcal{T}}$ is driven by the same Itô SDE as in \citet[Theorem~4.2]{park2025integration}. Therefore, the remainder of the proof in \citet[Theorem~4.2]{park2025integration} carries over unchanged, and the claim follows.
\end{proof}

\section{Proof of Theorem~\ref{variance_comparison}}
\label{proof_variance_comparison}

\begin{restatetheorem}{variance_comparison}{}
Define $\alpha = 2/M - 1/(2M_1) - 1/(2M_2)$ and $\beta = 1/(2M^2) - 1/(4M_1M_2)$. Assume $\beta > 0$ and $\alpha \geq 4/(3M + \beta M^4)$. Then, as $\Delta t \to 0$ and $\tau \to 0$, the estimator variances satisfy
\begin{equation}
    \mathbb{V}\left[ \hat{\ell}^{M_1, M_2}_{\mathrm{UEM}} \right] \le \mathbb{V}\left[ \hat{\ell}^M_{\mathrm{SG}} \right] = \mathbb{V}\left[ \hat{\ell}^M_{\mathrm{SEM}} \right] \le \mathbb{V}\left[ \hat{\ell}_{\mathrm{EM}} \right].
\end{equation}
\end{restatetheorem}

\begin{proof}
To establish the variance comparison, we begin by deriving explicit variance expansions for the four estimators under consideration. We use the same notation and conventions as in the preceding bias analyses and introduce no additional notation.

\begin{align}
    &\mathbb{V}\left[\hat{\ell}_{\mathrm{EM}}(\theta, n, x)\right] \\
    =~& \mathbb{E}\left[ \left( [\mathcal{L}[u_\theta] - \phi_{u_\theta}](t_n, x) + \frac{1}{2} [\xi_n^T H_{u_\theta} \xi_n - \mathrm{Tr}[H_{u_\theta}]](t_n, x) \right)^4 \right]  \\
    &- \left( ([\mathcal{L} [u_\theta] - \phi_{u_\theta}](t_n, x))^2 + \frac{1}{2} \mathrm{Tr}[H_{u_\theta}^2](t_n, x) \right)^2 + O(\Delta t^{1/2}) \\
    =~&([\mathcal{L}[u_\theta] - \phi_{u_\theta}](t_n, x))^4 + \frac{3}{2} ([\mathcal{L}[u_\theta] - \phi_{u_\theta}](t_n, x))^2 \mathbb{E}\left[([\xi_n^T H_{u_\theta} \xi_n - \mathrm{Tr}[H_{u_\theta}]](t_n, x))^2\right] \\
    &+\frac{1}{2}[\mathcal{L}[u_\theta] - \phi_{u_\theta}](t_n, x) \mathbb{E}\left[([\xi_n^T H_{u_\theta} \xi_n - \mathrm{Tr}[H_{u_\theta}]](t_n, x))^3\right] + \frac{1}{16}\mathbb{E}\left[ ([\xi_n^T H_{u_\theta} \xi_n - \mathrm{Tr}[H_{u_\theta}]](t_n, x))^4 \right] \\
    &- ([\mathcal{L}[u_\theta] - \phi_{u_\theta}](t_n, x))^4 - ([\mathcal{L}[u_\theta] - \phi_{u_\theta}](t_n, x))^2 \mathrm{Tr}[H_{u_\theta}^2](t_n, x) - \frac{1}{4}(\mathrm{Tr}[H_{u_\theta}^2](t_n, x))^2 + O(\Delta t^{1/2}) \\
    =~& 2([\mathcal{L}[u_\theta] - \phi_{u_\theta}](t_n, x))^2 \mathrm{Tr}[H_{u_\theta}^2](t_n, x) + 4[\mathcal{L}[u_\theta] - \phi_{u_\theta}](t_n, x) \mathrm{Tr}[H_{u_\theta}^3](t_n, x) + 3\mathrm{Tr}[H_{u_\theta}^4](t_n, x) \\
    &+ \frac{1}{2}(\mathrm{Tr}[H_{u_\theta}^2](t_n, x))^2 + O(\Delta t^{1/2}),
\end{align}

\begingroup
\allowdisplaybreaks[1]

\begin{align}
    &\mathbb{V}\left[ \hat{\ell}_{\mathrm{SG}}^{M}(\theta, n, x) \right] \\
    =~& \mathbb{E}\left[ \left(  [\mathcal{L}[u_\theta] - \phi_{u_\theta}](t_n, x) + \frac{1}{2M}\sum_{i=1}^{M}[\xi_{n,i}^T H_{u_\theta} \xi_{n,i} - \mathrm{Tr}[H_{u_\theta}]](t_n, x)\right)^4  \right] \\
    &- \left( [\mathcal{L}[u_\theta] - \phi_{u_\theta}](t_n, x))^2 + \frac{1}{2M}\mathrm{Tr}[H_{u_\theta}^2](t_n, x) \right)^2 + O(\tau) \displaybreak[2] \\
    =~&([\mathcal{L}[u_\theta] - \phi_{u_\theta}](t_n, x))^4 + \frac{3}{2M^2} ([\mathcal{L}[u_\theta] - \phi_{u_\theta}](t_n, x))^2 \sum_{i=1}^{M} \mathbb{E}\left[([\xi_{n,i}^T H_{u_\theta} \xi_{n,i} - \mathrm{Tr}[H_{u_\theta}]](t_n, x))^2\right] \\
    &+\frac{1}{2M^3}[\mathcal{L}[u_\theta] - \phi_{u_\theta}](t_n, x) \sum_{i=1}^{M} \mathbb{E}\left[([\xi_{n,i}^T H_{u_\theta} \xi_{n,i} - \mathrm{Tr}[H_{u_\theta}]](t_n, x))^3\right] \\
    &+ \frac{1}{16M^4} \sum_{i=1}^{M} \mathbb{E}\left[ ([\xi_{n,i}^T H_{u_\theta} \xi_{n,i} - \mathrm{Tr}[H_{u_\theta}]](t_n, x))^4 \right] \\
    &+ \frac{3}{8M^4}\sum_{1\le i < j\le M} \mathbb{E}\left[[(\xi_{n,i}^T H_{u_\theta} \xi_{n,i} - \mathrm{Tr}[H_{u_\theta}])^2(\xi_{n,j}^T H_{u_\theta} \xi_{n,j} - \mathrm{Tr}[H_{u_\theta}])^2](t_n, x) \right] \\
    &- ([\mathcal{L}[u_\theta] - \phi_{u_\theta}](t_n, x))^4 - \frac{1}{M}([\mathcal{L}[u_\theta] - \phi_{u_\theta}](t_n, x))^2 \mathrm{Tr}[H_{u_\theta}^2](t_n, x) - \frac{1}{4M^2}(\mathrm{Tr}[H_{u_\theta}^2](t_n, x))^2 + O(\tau) \\ 
    =~& \frac{2}{M}([\mathcal{L}[u_\theta] - \phi_{u_\theta}](t_n, x))^2 \mathrm{Tr}[H_{u_\theta}^2](t_n, x) + \frac{4}{M^2}[\mathcal{L}[u_\theta] - \phi_{u_\theta}](t_n, x) \mathrm{Tr}[H_{u_\theta}^3](t_n, x) + \frac{3}{M^3}\mathrm{Tr}[H_{u_\theta}^4](t_n, x) \\
    &+ \frac{1}{2M^2}(\mathrm{Tr}[H_{u_\theta}^2](t_n, x))^2 + O(\tau), \\
\end{align}
\endgroup
\begin{align}
    &\mathbb{V}\left[ \hat{\ell}_{\mathrm{SEM}}^{M}(\theta, n, x) \right] \\
    =~& \mathbb{E}\left[ \left(  [\mathcal{L}[u_\theta] - \phi_{u_\theta}](t_n, x) + \frac{1}{2M}\sum_{i=1}^{M}[\xi_{n,i}^T H_{u_\theta} \xi_{n,i} - \mathrm{Tr}[H_{u_\theta}]](t_n, x)\right)^4  \right] \\
    &- \left( ([\mathcal{L}[u_\theta] - \phi_{u_\theta}](t_n, x))^2 + \frac{1}{2M}\mathrm{Tr}[H_{u_\theta}^2](t_n, x) \right)^2 + O(\Delta t^{1/2}) \\
    =~& \frac{2}{M}([\mathcal{L}[u_\theta] - \phi_{u_\theta}](t_n, x))^2 \mathrm{Tr}[H_{u_\theta}^2](t_n, x) + \frac{4}{M^2}[\mathcal{L}[u_\theta] - \phi_{u_\theta}](t_n, x) \mathrm{Tr}[H_{u_\theta}^3](t_n, x) + \frac{3}{M^3}\mathrm{Tr}[H_{u_\theta}^4](t_n, x) \\
    &+ \frac{1}{2M^2}(\mathrm{Tr}[H_{u_\theta}^2](t_n, x))^2 + O(\Delta t^{1/2}), \\
\end{align}
\begin{align}
    &\mathbb{V}\left[ \hat{\ell}^{M_1, M_2}_{\mathrm{UEM}}(\theta, n, x) \right] \\
    =~& \mathbb{E}\left[ \left(  [\mathcal{L}[u_\theta] - \phi_{u_\theta}](t_n, x) + \frac{1}{2M_1}\sum_{i=1}^{M_1}[\xi_{n,i}^T H_{u_\theta} \xi_{n,i} - \mathrm{Tr}[H_{u_\theta}]](t_n, x)\right)^2  \right. \\
    &\left.\left(  [\mathcal{L}[u_\theta] - \phi_{u_\theta}](t_n, x) + \frac{1}{2M_2}\sum_{i=M_1+1}^{M_1+M_2}[\xi_{n,i}^T H_{u_\theta} \xi_{n,i} - \mathrm{Tr}[H_{u_\theta}]](t_n, x)\right)^2\right]  \\
    & -([\mathcal{L}[u_\theta] - \phi_{u_\theta}](t_n, x))^4 + O(\Delta t^{1/2}) \\
    =~& \left( ([\mathcal{L}[u_\theta] - \phi_{u_\theta}](t_n, x))^2 + \frac{1}{2M_1} \mathrm{Tr}[H_{u_\theta}^2] \right) \left( ([\mathcal{L}[u_\theta] - \phi_{u_\theta}](t_n, x))^2 + \frac{1}{2 M_2} \mathrm{Tr}[H_{u_\theta}^2] \right) \\
    & -([\mathcal{L}[u_\theta] - \phi_{u_\theta}](t_n, x))^4 + O(\Delta t^{1/2}) \\
    =~& \left( \frac{1}{2M_1} + \frac{1}{2M_2} \right) ([\mathcal{L}[u_\theta] - \phi_{u_\theta}](t_n, x))^2 \mathrm{Tr}[H_{u_\theta}^2] + \frac{1}{4M_1 M_2}(\mathrm{Tr}[H_{u_\theta}^2])^2 + O(\Delta t^{1/2}).
    \label{Un-EM-BSDE_variance}
\end{align}

\newpage
Comparing the variance expansions above, we first observe that, as $\Delta t\to0$ and $\tau\to0$,

\begin{equation}
    \mathbb{V}\left[ \hat{\ell}_{\mathrm{SG}}^{M} \right] = \mathbb{V}\left[ \hat{\ell}^{M}_{\mathrm{SEM}} \right]
\end{equation}

Therefore, it remains only to prove
\begin{equation}
    \mathbb{V}\left[ \hat{\ell}^{M}_{\mathrm{SEM}} \right] \le \mathbb{V}\left[ \hat{\ell}_{\mathrm{EM}} \right] \quad \text{and} \quad \mathbb{V}\left[ \hat{\ell}^{M_1, M_2}_{\mathrm{UEM}} \right] \le \mathbb{V}\left[ \hat{\ell}^{M}_{\mathrm{SG}} \right].
\end{equation}

Let $y := [\mathcal{L}[u_\theta] - \phi_{u_\theta}](t_n, x)$, $H := H_{u_\theta}(t_n, x)$, and $T_k := \mathrm{Tr}[H^k]$. Under the variance expansions derived above, the two desired inequalities are equivalent to the nonnegativity of the following two quadratic polynomials in $y$:
\begin{align}
    Q_1(y) &:= 2\left(1 - \frac{1}{M}\right) y^2 T_2 + 4\left(1 - \frac{1}{M^2} \right)y T_3 + 3\left(1 - \frac{1}{M^3}\right)T_4 + \frac{1}{2}\left(1 - \frac{1}{M^2} \right) T_2^2 \geq 0, \\
    Q_2(y) &:= \alpha y^2 T_2 + \frac{4}{M^2} y T_3 + \frac{3}{M^3}T_4 + \beta T_2^2 \geq 0,
\end{align}
Here, $Q_1(y) \geq 0$ corresponds to $\mathbb{V}[\hat{\ell}_{\mathrm{SEM}}^M]\le
\mathbb{V}[\hat{\ell}_{\mathrm{EM}}]$, whereas $Q_2(y) \geq 0$ corresponds to $\mathbb{V}[\hat{\ell}_{\mathrm{UEM}}^{M_1,M_2}]\le
\mathbb{V}[\hat{\ell}_{\mathrm{SG}}^M]$.  

For $Q_1$, the case $M=1$ gives $Q_1 \equiv 0$. For $M > 1$, the leading coefficient is nonnegative, and it suffices to show that its discriminant is nonpositive. The discriminant is
\begin{equation}
    D_1 = 16\left( 1 - \frac{1}{M^2} \right)^2 T_3^2 - 8\left(1 - \frac{1}{M}\right) T_2 \left[ 3\left(1 - \frac{1}{M^3}\right) T_4 + \frac{1}{2}\left(1 - \frac{1}{M^2}\right) T_2^2 \right]
\end{equation}
For $Q_2$, since $\beta > 0$ and $\alpha \geq 4/(3M + \beta M^4) > 0$, its leading coefficient is nonnegative. Its discriminant is
\begin{equation}
    D_2 = \left( \frac{4}{M^2} T_3 \right)^2 - 4\alpha T_2 \left(\frac{3}{M^3}T_4 + \beta T_2^2 \right).
\end{equation}

Since $H$ is symmetric, let $(\lambda_1, ..., \lambda_d)$ be its real eigenvalues. Then
\begin{align}
    T_2^2 &= \left( \sum_{i=1}^{d} \lambda_i^2 \right)^2 = \sum_{i=1}^{d} \lambda_i^4 + 2 \sum_{1\le i < j \le d}\lambda_i^2 \lambda_j^2 \geq \sum_{i=1}^{d}\lambda_i^4 = T_4, \\
    T_3^2 &= \left( \sum_{i=1}^{d} (\lambda_i)(\lambda_i)^2 \right)^2 \le \left( \sum_{i=1}^{d} \lambda_i^2\right) \left( \sum_{i=1}^{d} (\lambda_i^2)^2 \right) = T_2 T_4 & [~\because~ \text{Cauchy-Schwarz}~].
\end{align}

Using these inequalities and $T_2 \geq 0$, we obtain
\begin{align}
    D_1 &\le 16\left(1 - \frac{1}{M^2}\right)^2 T_2 T_4 - 24\left(1 - \frac{1}{M} \right)\left(1 - \frac{1}{M^3} \right) T_2 T_4 \\
    &=-8\left(1 - \frac{1}{M}\right)\left(1 + \frac{1}{M} + \frac{5}{M^2} + \frac{2}{M^3}\right) T_2 T_4 \le 0. \\
    D_2 &\le \frac{16}{M^4}T_2 T_4 - 4\alpha\left(\frac{3}{M^3} + \beta \right)T_2 T_4 \le 0
\end{align}
where the last inequality follows from $\alpha \geq 4/(3M + \beta M^4) = (4/M^4)/(3/M^3 + \beta)$. Therefore $Q_1(y) \geq 0$ and $Q_2(y) \geq 0$ for all $y$, which proves $\mathbb{V}\left[ \hat{\ell}^{M}_{\mathrm{SEM}} \right] \le \mathbb{V}\left[ \hat{\ell}_{\mathrm{EM}} \right]$ and $\mathbb{V}\left[ \hat{\ell}^{M_1, M_2}_{\mathrm{UEM}} \right] \le \mathbb{V}\left[ \hat{\ell}^{M}_{\mathrm{SG}} \right]$. This concludes the proof.

\end{proof}

\section{Computation of $\mathbb{E}\left[ \left(\xi^T H \xi - \mathrm{Tr}[H] \right)^n \right]$}

Suppose $\xi\sim \mathcal{N}(0, I_d)$ and $H \in \mathbb{R}^{d\times d}$ is a symmetric matrix. Then, by the spectral theorem, we can write $H = U^T \Lambda U$ where $U$ is an orthogonal matrix and $\Lambda = \mathrm{diag}(\lambda_1, ..., \lambda_d)$. If we define $Z:= U\xi$, then $Z \sim \mathcal{N}(0, I_d)$, because of rotation invariance. Hence we obtain
\begin{equation}
    \xi^T H \xi = (U\xi)^T \Lambda (U\xi) = \sum_{i=1}^{d} \lambda_i Z_i^2.
\end{equation}
Thus, we can denote
\begin{equation}
    X := \xi^T H \xi - \mathrm{Tr}[H] = \sum_{i=1}^{d} \lambda_i(Z_i^2 - 1).
\end{equation}
Note that the moment generating function of $X$ is $M_X(t) := \mathbb{E}[e^{tX}]$, and the cumulant generating function is
\begin{equation}
    K_X(t) := \log M_X(t) = \sum_{r\geq 1}\frac{\kappa_r t^r}{r!}.
    \label{CGF}
\end{equation}
where the coefficients $\kappa_r$ are the cumulants. We can compute $\mathbb{E}[X^n]$ using the following:
\begin{align}
    \mathbb{E}[X^n] 
    &= \left.\frac{d^n}{dt^n} M_X(t)\right|_{t = 0} = \left.\frac{d^n}{dt^n} e^{K_X(t)}\right|_{t=0} \\ 
    &= e^{K_X(0)} \sum_{\pi \in \Pi_n} \prod_{B \in \pi} K_X^{(|B|)}(0) & [~\because \text{Faà di Bruno's formula}~] \\
    &= \sum_{\pi \in \Pi_n} \prod_{B \in \pi} \kappa_{|B|},
    \label{E[X^n]}
\end{align}
where $\Pi_n$ denotes the set of all partitions of the set $\{1, ..., n\}$. \\

Since $Z_i^2 \sim \chi_1^2$ so that $\log \mathbb{E}[e^{tZ_i^2}] = -\frac{1}{2} \log(1 - 2t)$, we obtain
\begin{align}
    K_X(t)
    &= \log \mathbb{E}\left[ \exp\left( t \sum_{i=1}^{d} \lambda_i(Z_i^2 - 1)\right) \right] \\
    &= \log \prod_{i=1}^{d} e^{-t\lambda_i} \mathbb{E}\left[ e^{t\lambda_i Z_i^2} \right] & [~\because \text{independence}~] \\
    &= \sum_{i=1}^{d}\left( - t \lambda_i - \frac{1}{2} \log(1 - 2t\lambda_i) \right).
\end{align}

Since $\log(1 - 2t \lambda_i) = - \sum_{r \geq 1} \frac{(2t \lambda_i)^r}{r}$ for $|2t\lambda| < 1$, we obtain
\begin{equation}
    K_X(t) = \sum_{i=1}^{d} \sum_{r \geq 2} \frac{(2t\lambda_i)^r}{2r} = \sum_{r\geq 2} \frac{2^{r-1}}{r}t^r \sum_{i=1}^{d}\lambda_i^r.
\end{equation}

Matching coefficients with \eqref{CGF} yields
\begin{equation}
    \kappa_1 = 0,~~~~ \kappa_r = 2^{r-1}(r-1)! \mathrm{Tr}[H^r] ~~\text{for}~r\geq 2.
\end{equation}

Therefore we can compute $\mathbb{E}[X^n]$ using \eqref{E[X^n]}, such as
\begin{align}
    \label{E[X]}
    \mathbb{E}[X] &= \kappa_1 = 0. \\
    \label{E[X^2]}
    \mathbb{E}[X^2] &= \kappa_1^2 + \kappa_2 = 2\mathrm{Tr}[H^2]. \\
    \label{E[X^3]}
    \mathbb{E}[X^3] &= \kappa_1^3 + 3\kappa_1 \kappa_2 + \kappa_3 = 8\mathrm{Tr}[H^3]. \\ 
    \label{E[X^4]}
    \mathbb{E}[X^4] &= \kappa_1^4 + 6\kappa_1^2 \kappa_2 + 3\kappa_2^2 + 4\kappa_1 \kappa_3 + \kappa_4 = 48\mathrm{Tr}[H^4] + 12(\mathrm{Tr}[H^2])^2.
\end{align}

\section{Experimental Details}

We employ deep neural networks to approximate the solution for all benchmark problems. Following the experimental setup in \citet{ye2024fbsjnn}, the model for the partial integro-differential equation consists of two hidden layers with 256 neurons each and utilizes the Leaky-ReLU activation function. It is trained for 10,000 iterations using the Adam optimizer with a piecewise-constant scheduler. For the remaining benchmarks, we adopt a deeper configuration. These models comprise four hidden layers with 512 neurons each and the Mish activation function. Training is conducted for 100,000 iterations using the Adam optimizer with a cosine decay scheduler and an initial learning rate of $10^{-3}$. We conduct experiments under two different settings for terminal conditions, namely hard and soft constraints. For the stochastic simulations, we use a trajectory batch size of 64, translating to 64 realizations of the underlying Brownian motions. All simulations are implemented using the JAX library with 32-bit floating-point precision (float32) and performed on an NVIDIA RTX 6000 Ada Generation GPU, except for the experiment in Section \ref{debiased_one_step_losses}, which uses 64-bit floating-point precision (float64) on an NVIDIA A100 GPU due to the increased numerical sensitivity of replacing Shotgun's one-step loss with the debiased loss.

We fix $N=100$ for all experiments to maintain consistent input dimensions, except for the Shotgun method. For Shotgun, we strictly follow the original literature by setting $\tau = 4^{-5}$ and $N=10$. Despite the shorter trajectory, the inherent $\delta$ parameter provides sufficient regularization to compensate for the reduced discretization, thereby preserving both gradient flow and accuracy.

\subsection{Relative $L_{2}$ error metric}
To evaluate performance, test trajectories are generated using the exact solution when available; otherwise, they are constructed via a Brownian Bridge to ensure high-fidelity forward trajectories. The reported accuracy is the average relative $L_{2}$ error (RL2) calculated across 256 test trajectories, defined as:

\begin{equation}
    \text{RL2} = \sqrt{\frac{\sum_{n=0}^{N}|u(t_n, X_n)-u_\theta(t_n, X_n)|^2}{\sum_{n=0}^{N}|u(t_n, X_n)|^2}}
\end{equation}

\subsection{Benchmark equations}
\label{Benchmark_Equations}
In this section, we focus on benchmark PDEs within the general class given in \eqref{PDE}.

\textbf{Hamilton--Jacobi--Bellman Equation (HJB).} We consider the $d$-dimensional quadratic HJB equation. For $t \in [0, T_{\text{end}}]$ and $x \in \mathbb{R}^{d}$, the PDE is defined as:
\begin{equation}
    \partial_t u(t, x) + \mathrm{Tr}[\nabla^2 u(t, x)] = \|\nabla u(t, x)\|^{2}.
\end{equation}
with the terminal condition $u(T_{\text{end}}, x) = g(x)$. By \eqref{FBSDE}, the above PDE admits the following FBSDE formulation:
\begin{align}
dX_{t} &= \sqrt{2} dW_{t}.  \\
dY_{t} &= \|Z_{t}\|^{2} dt + \sqrt{2} Z_{t}^T dW_{t}.
\end{align}
where $X_0 = \xi \in \mathbb{R}^d$ and $Y_{T_{\text{end}}} = g(X_{T_{\text{end}}})$. For our experiments, we set $d=100$, $T_{\text{end}}=1$, $\xi = (0,0,...,0) \in \mathbb{R}^{100}$, and $g(x)=\ln(0.5(1+\|\mathbf{x}\|^{2}))$. The exact solution to this equation is given by 
\begin{equation}
    u(t,x)=-\ln \left (\mathbb{E}\left [\exp \left (-g(x+\sqrt{2}W_{T_{\text{end}}-t}\right )\right ]\right ).
\end{equation}

\textbf{Black--Scholes--Barenblatt Equation (BSB).} We consider the $d$-dimensional BSB equation:
\begin{equation}
    \partial_t u + \frac{1}{2}\mathrm{Tr[\alpha^2 \operatorname{diag}(x^2) \nabla^2 u(t, x)]} = r(u(t, x) - \nabla u(t, x)^T x).
\end{equation}

where $t \in [0,T_{\text{end}}]$ and $x \in \mathbb{R}^{d}$, with the terminal condition $u(T_{\text{end}}, x) = g(x)$. Applying \eqref{FBSDE} yields the following FBSDE formulation:
\begin{align}
    dX_{t} &= \alpha \operatorname{diag}(X_{t}) dW_{t} \\
    dY_{t} &= r(Y_{t} - Z_{t}^T X_{t}) dt + \alpha Z_{t}^T \operatorname{diag}(X_{t}) dW_{t},
\end{align}
where $X_0 = \xi \in \mathbb{R}^d$ and $Y_{T_{\text{end}}} = g(X_{T_{\text{end}}})$. For our experiments, we set $d = 100$, $T_{\text{end}} = 1$, $\xi = (1, 0.5, 1, 0.5, ..., 1, 0.5)\in\mathbb{R}^{d}$, $\alpha = 0.4$, $r = 0.05$, and $g(x) = \|x\|^2$. The exact solution to this equation is given by
\begin{equation}
    u(t,x) = \exp \left( (r + \alpha^2)(T_{\text{end}}-t) \right)g(x).
\end{equation}

\textbf{Allen--Cahn Equation (AC)}.
We consider the $d$-dimensional AC equation, which is a semilinear parabolic PDE. For $t \in [0,T_{\text{end}}]$ and $x \in \mathbb{R}^{d}$, the equation is defined as:
\begin{equation}
    \frac{\partial u}{\partial t} + \frac{1}{2} \Delta u = u^3 - u,
\end{equation}
with the terminal condition $u(T_{\text{end}},x)=g(x)$. From \eqref{FBSDE}, we have:
\begin{align}
    dX_{t} &= dW_{t}, \\
    dY_{t} &= (Y^3_{t} - Y_{t})dt + Z_{t}^T dW_{t},
\end{align}
where $X_0 = \xi \in \mathbb{R}^d$ and $Y_{T_{\text{end}}} = g(X_{T_{\text{end}}})$. For our experiments, we set $d = 20$, $T_{\text{end}} = 0.3$, $\xi = (0,0,...,0) \in \mathbb{R}^{20}$, and $g(x) = (2+0.4\|x\|^{2})^{-1}$.

Due to the absence of an exact solution for the AC equation, the relative error in Table~\ref{RL2_RLT0} was calculated at $t=0$ by comparing the model's output with a reference value of 0.30879. This benchmark value was obtained via the Branching Diffusion Method \citep{MR3993178}, which is considered an unbiased approach for this problem.

\subsection{Fully-Coupled FBSDE (BZ)}
We consider the following non-linear PDE:
\begin{equation}
    \partial_{t}u(t, x) + \frac{1}{2} \mathrm{Tr} [\sigma (t,x,u)^{T} \nabla^{2}u(t, x) \sigma (t,x,u)] = \phi(t,x,u),
\end{equation}
with the terminal condition $u(T_{\text{end}}, x) = g(x)$, where the respective terms are given by:
\begin{align}
    \sigma (t,x,u) &= \alpha u(t,x) I_{d}, \\
    \phi(t,x,u) &= ru - \frac{1}{2}e^{-3r(T_{\text{end}}-t)} \alpha^{2} \left( D \sum_{j=1}^{d} \sin{(x_{j})} \right)^{3}.
\end{align}
The PDE can be reformulated as the following fully-coupled FBSDE:
\label{BZ_definition}
\begin{align}
    dX_{t} &=\alpha Y_{t}dW_{t}, \\
    dY_{t} &=  \left[r Y_{t} - \frac{1}{2} e^{-3r(T_{\text{end}}-t)} \alpha^{2} \left( D \sum_{j=1}^{d} \sin(X_{j,t}) \right)^3 \right] dt + Z_t^T \sigma(t, X_t, Y_t) d W_{t},
\end{align}
where $X_0 = \xi \in \mathbb{R}^d$ and $Y_{T_{\text{end}}} = g(X_{T_{\text{end}}})$. For our experiments, we set $d=100$, $T_{\text{end}}=1$, $r=0.1$, $\alpha=0.3$, $D=0.1$, $\xi = (\pi/2, \pi/2, ..., \pi/2)$, and $g(x) = D\sum_{j=1}^{d} \sin(x_j)$. The above FBSDE has the exact solution: 
\begin{equation}
    u(t,x) = e^{-r(T_{\text{end}}-t)} D \sum_{j=1}^{d} \sin (x_j). 
\end{equation}

The BZ is introduced to evaluate the algorithm's performance on a fully-coupled system, where the forward and backward dynamics are interdependent. This formulation is adapted from \citet{bender2008time} and has been utilized as a benchmark for high-dimensional BSDE solvers, as mentioned in \citet{park2025integration}.

\subsection{Partial integro-differential equation (PIDE)}
\label{PIDE_definition}
We consider a high-dimensional PIDE that incorporates jump components. For $t\in [0, T_{\text{end}}]$ and $x \in \mathbb{R}^d$, the governing equation is defined as follows:
\begin{equation}
\begin{cases}
    \dfrac{\partial u}{\partial t}(t, x) + \dfrac{1}{2} \operatorname{Tr}(\tau^2 \nabla_x^2 u(t, x)) + \left\langle \dfrac{1}{2} \epsilon x, \nabla_x u(t, x) \right\rangle \\
    \quad + \displaystyle \int_{\mathbb{R}} (u(t, x+e) - u(t, x) - \langle e, \nabla_x u(t, x) \rangle) \nu(de)  = \lambda(\mu_\phi^2 + \sigma_\phi^2) + \tau^2 + \dfrac{\epsilon}{d} \|x\|^2, \\
    u(T_{\text{end}}, x) = \dfrac{1}{d} \|x\|^2,
\end{cases}
\end{equation}
where $\nu(de)$ denotes the Lévy measure. By applying the non-linear Feynman-Kac formula for jump-diffusion processes, this PIDE can be transformed into a system of FBSDEs:
\begin{align}
    dX_t &= \frac{1}{2} \epsilon X_t dt + \tau dW_t + \int_{\mathbb{R}} e \tilde{\mu}(dt, de),\\
    dY_t &= \left( \lambda (\mu_\phi^2 + \sigma_\phi^2) + \tau^2 + \frac{\epsilon}{d}\|X_t\|^2 \right) dt + \tau Z_t^T dW_t + \int_{\mathbb{R}} \hat{u}_t(e) \tilde{\mu}(dt, de),
\end{align}
where $\tilde{\mu}(dt,de)$ is the compensated Poisson random measure.
Following the Euler-Maruyama discretization, the corresponding FBSDE is given by:
\begin{equation}
\begin{cases}
    \displaystyle X_{n+1} = X_{n} + \dfrac{1}{2} \epsilon X_{n} \Delta t + \tau \Delta W_{n} + \sum_{i=1}^{\mu_{n}}z_{i} - \lambda \mu_{\phi} \Delta t, \\
    \displaystyle Y_{n+1} = Y_{n} + \left( \lambda(\mu_{\phi}^2 + \sigma_{\phi}^2) + \tau^2 + \frac{\epsilon}{d}\|X_n\|^2 \right)\Delta t + \tau Z_n^T \Delta W_n + \sum_{i=1}^{\mu_n}\hat{u}_n(z_i) - \lambda \mu_{\phi} Z_n^T \mathbf{1} \Delta t
\end{cases}
\end{equation}
where $\mu_{n}$ represents the number of jump occurrences within the interval $[t_{n},t_{n+1})$, which is a Poisson-distributed random variable with intensity $\lambda \Delta t$. For each jump $i=1, ..., \mu_{n}$, the jump size $z_{i}$ is an independent random variable drawn from the probability distribution $\phi(e)$, modeled as a normal distribution $\mathcal{N}(\mu_{\phi}, \sigma_{\phi}^{2})$. For our experiments, we set $d = 100$, $T_{\text{end}} = 1$, $\epsilon = 0.1$, $\tau = 0.1$, $\lambda = 0.01$, $\mu_{\phi} = 0.01$, and $\sigma_{\phi} = 0.01$. The true solution to this equation is $u(t,x) = \|x\|^{2}/d$.

We include the PIDE to test how the algorithm handles jump-diffusion processes, which are not present in standard PDEs. Although this specific case was not the primary focus during our scheme's development, we use it as a benchmark to ensure the method remains effective across various mathematical structures, including non-local operators.

\section{Discussion on Efficient Higher-Order Differentiation}
\label{app:efficient-higher-order-differentiation}

Throughout the main experiments, all derivatives are computed using reverse-mode automatic differentiation. This provides a simple and uniform implementation across all methods, but it does not necessarily represent the most efficient way to compute higher-order derivatives. In parallel, several recent works have proposed more efficient automatic-differentiation procedures for higher-order differential operators \citep{shi2024stochastic, dangel2025collapsing}.

Among these directions, Forward Laplacian provides an efficient exact procedure for computing neural-network Laplacians \citep{li2024computational}. To examine the effect of such implementation-level acceleration, we additionally apply Forward Laplacian to the weighted-Laplacian computations required by FS-PINNs and Heun-BSDE. Compared with the reverse-mode automatic-differentiation implementation used in the main experiments, Forward Laplacian accelerates FS-PINNs and Heun-BSDE by approximately $3\times$. Nevertheless, both methods remain clearly slower than Un-EM-BSDE, indicating that specialized higher-order differentiation reduces but does not eliminate the computational overhead of second-order baselines.

\section{Further Experiments}

\subsection{Ablation on $M_1$ and $M_2$}
\label{ablation}
Our variance expression in $\eqref{Un-EM-BSDE_variance}$ consists of two nonnegative terms weighted by $1/M_1 + 1/M_2$ and $1/(M_1M_2)$. Fixing the total sampling budget $M_1 + M_2 = K$, we have $1/M_1 + 1/M_2 = K/(M_1M_2)$; thus the variance is minimized by maximizing $M_1M_2$, attained at $M_1 = M_2$ by AM--GM.

\begin{table*}[t]
\centering
    \captionsetup{justification=centering}
    \caption{Performance comparison for $M_1 = M_2 \in \{1,2,5,10,20,50\}$ under soft and hard constraints.
Values show mean RL2 error $(\times 10^{-2})$ with standard deviation $(\times 10^{-2})$ over three random seeds.} 
\label{ablation_table}
\centering
\small
\begin{tabular}{lcccccc}
\toprule
Cases & $M_1 = M_2 = 1$ & $M_1 = M_2 = 2$ & $M_1 = M_2 = 5$ & $M_1 = M_2 = 10$ & $M_1 = M_2 = 20$ & $M_1 = M_2 = 50$ \\
\midrule

\multicolumn{7}{c}{\textbf{Soft constraint}} \\
\midrule
HJB & $0.1464$\tiny{$\pm 0.0105$} & $0.1374$\tiny{$\pm 0.0032$} & $0.1348$\tiny{$\pm0.0153$} & $0.1406$\tiny{$\pm0.0100$} & $0.1433$\tiny{$\pm0.0167$} & $0.1316$\tiny{$\pm0.0131$} \\
BSB & $0.0805$\tiny{$\pm0.0049$} & $0.0829$\tiny{$\pm 0.0023$} & $0.0814$\tiny{$\pm 0.0026$} & $0.0842$\tiny{$\pm0.0063$} & $0.0869$\tiny{$\pm0.0023$} & $0.0821$\tiny{$\pm0.0008$}\\

\midrule
\multicolumn{7}{c}{\textbf{Hard constraint}} \\
\midrule

HJB & $0.3741$\tiny{$\pm0.0882$} & $0.1526$\tiny{$\pm0.0038$} & $0.1444$\tiny{$\pm0.0275$} & $0.1378$\tiny{$\pm0.0056$} & $0.1144$\tiny{$\pm0.0224$} & $0.1177$\tiny{$\pm0.0085$} \\
BSB & $0.0205$\tiny{$\pm0.0001$} & $0.0154$\tiny{$\pm0.0003$} & $0.0120$\tiny{$\pm0.0004$} & $0.0131$\tiny{$\pm0.0017$} & $0.0118$\tiny{$\pm0.0015$} & $0.0110$\tiny{$\pm0.0014$} \\

\bottomrule
\end{tabular}
\end{table*}

We therefore focus on balanced configurations and report results for $M_1 = M_2 \in \{1, 2, 5, 10, 20, 50\}$ in Table~\ref{ablation_table}. The performance is largely stable across these values, but $M_1 = M_2 = 5$ gives the most favorable accuracy--time balance overall. In particular, while increasing $M_1=M_2$ monotonically increases the per-iteration sampling cost, the wall-clock training time scales moderately with $M_1=M_2$: $M_1=M_2=1$ requires $1.52\times$ the training time of EM-BSDE,
$M_1=M_2=2$ requires $1.55\times$, $M_1=M_2=5$ requires $1.79\times$,
$M_1=M_2=10$ requires $2.22\times$, $M_1=M_2=20$ requires $3.15\times$, and $M_1=M_2=50$ requires $7.51\times$. Given that larger values provide limited additional gains in RL2 relative to their increased cost, we adopt $M_1 = M_2 = 5$ as the default setting throughout the main paper.

\subsection{Additional visualizations}
To further demonstrate the superior performance of our method, we provide additional visualizations of the learned trajectories for the HJB and BSB equations.
The visualizations for the HJB equation under soft and hard constraint settings are presented in Figures~\ref{fig:hjb_soft} and~\ref{fig:hjb_hard}, respectively. Similarly, the corresponding results for the BSB equation are illustrated in Figures~\ref{fig:bsb_soft} and~\ref{fig:bsb_hard}. In Figures~\ref{fig:bsb_soft} (BSB soft constraint), Deep Shotgun has been excluded from the visualization, as the model failed to converge and was unable to produce accurate predictions.
\begin{figure}
  \centering
  \includegraphics[width=0.9\columnwidth]{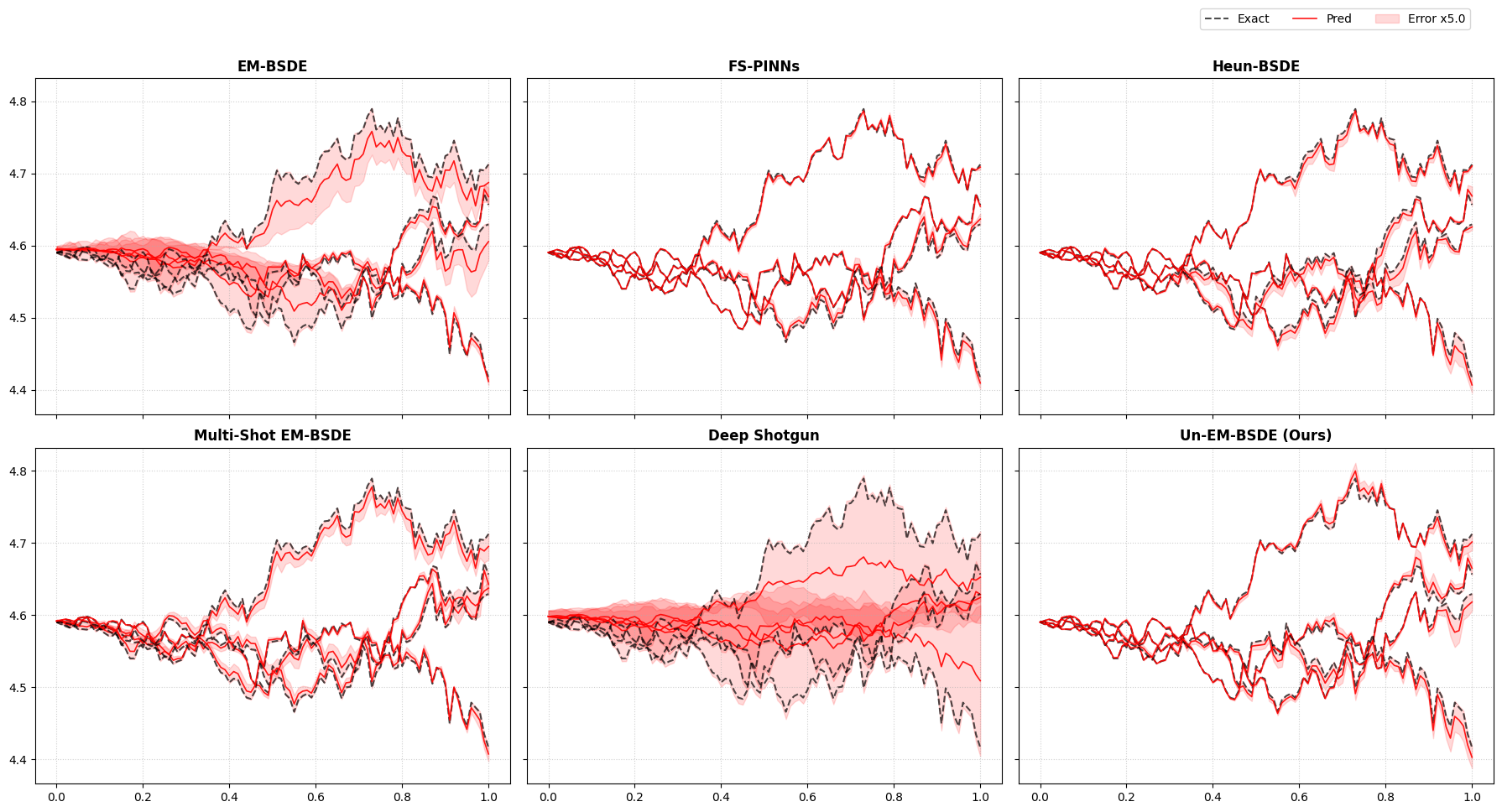}
  \caption{HJB($d=100$) path prediction with error band in the soft-constraint setting.}
  \label{fig:hjb_soft}
\end{figure}
\begin{figure}
  \centering
  \includegraphics[width=0.9\columnwidth]{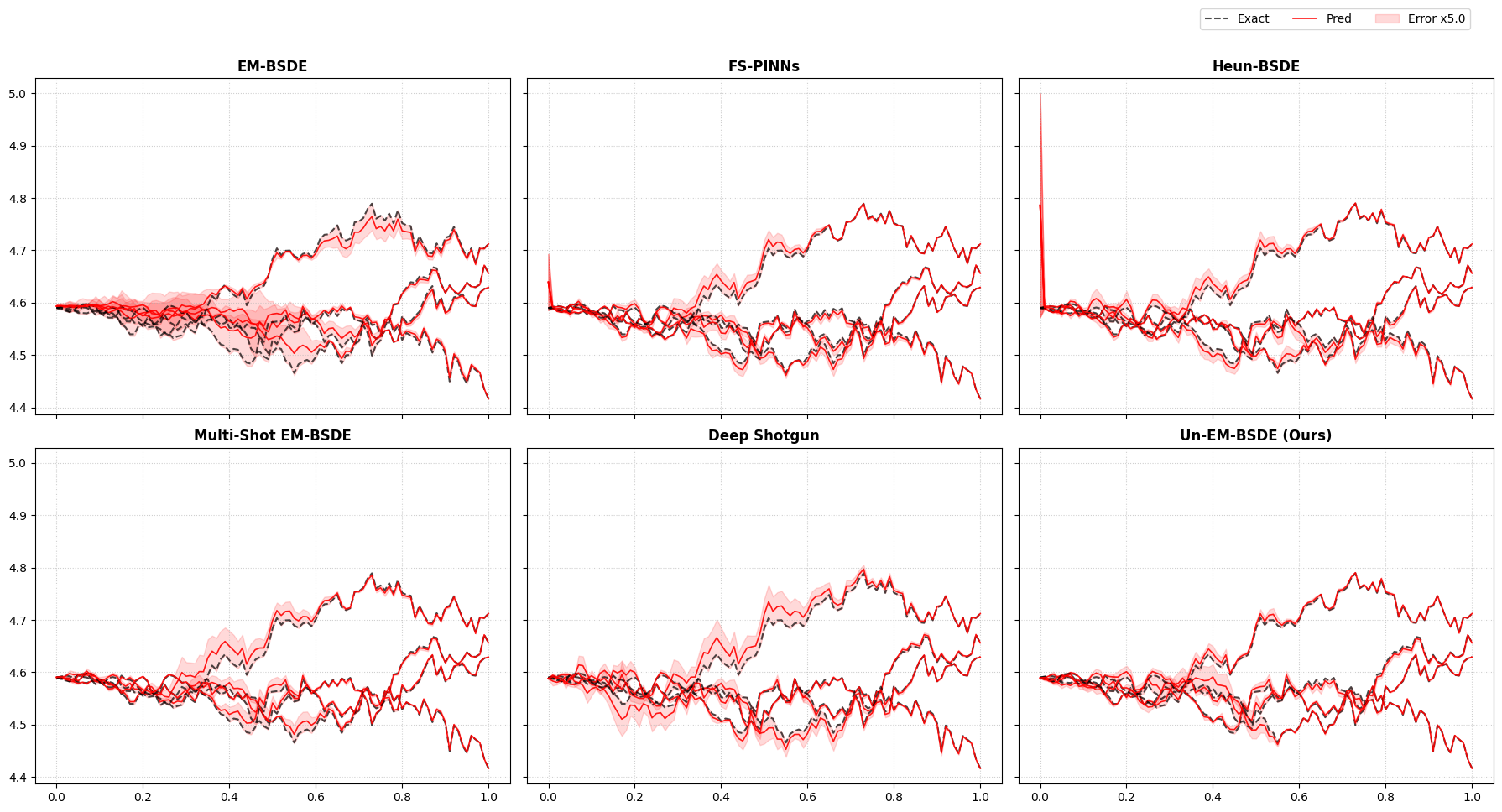}
  \caption{HJB($d=100$) path prediction with error band in the hard-constraint setting.}
  \label{fig:hjb_hard}
\end{figure}
\begin{figure}
  \centering
  \includegraphics[width=0.9\columnwidth]{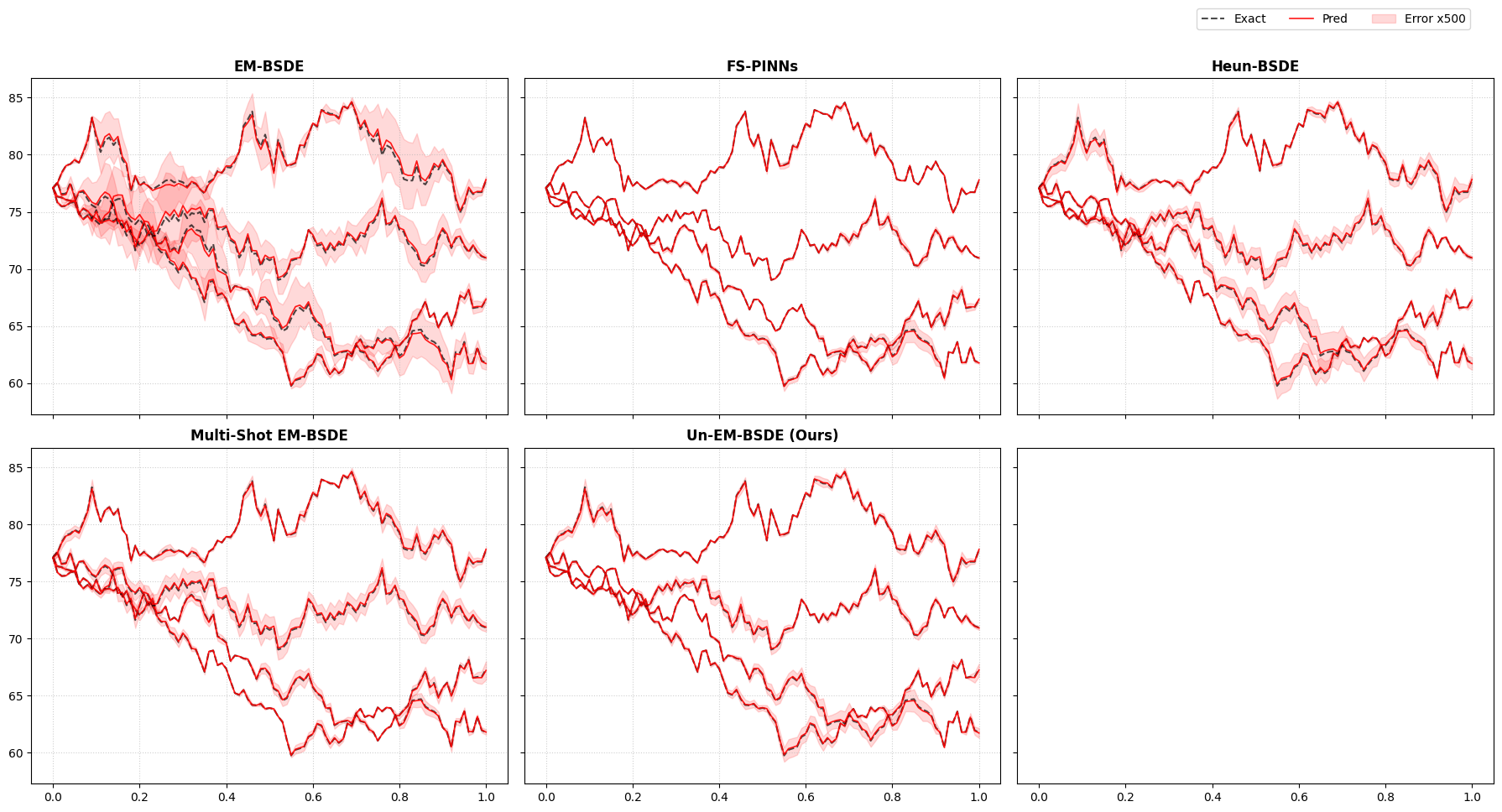}
  \caption{BSB($d=100$) path prediction with error band in the soft-constraint setting.}
  \label{fig:bsb_soft}
\end{figure}
\begin{figure}
  \centering
  \includegraphics[width=0.9\columnwidth]{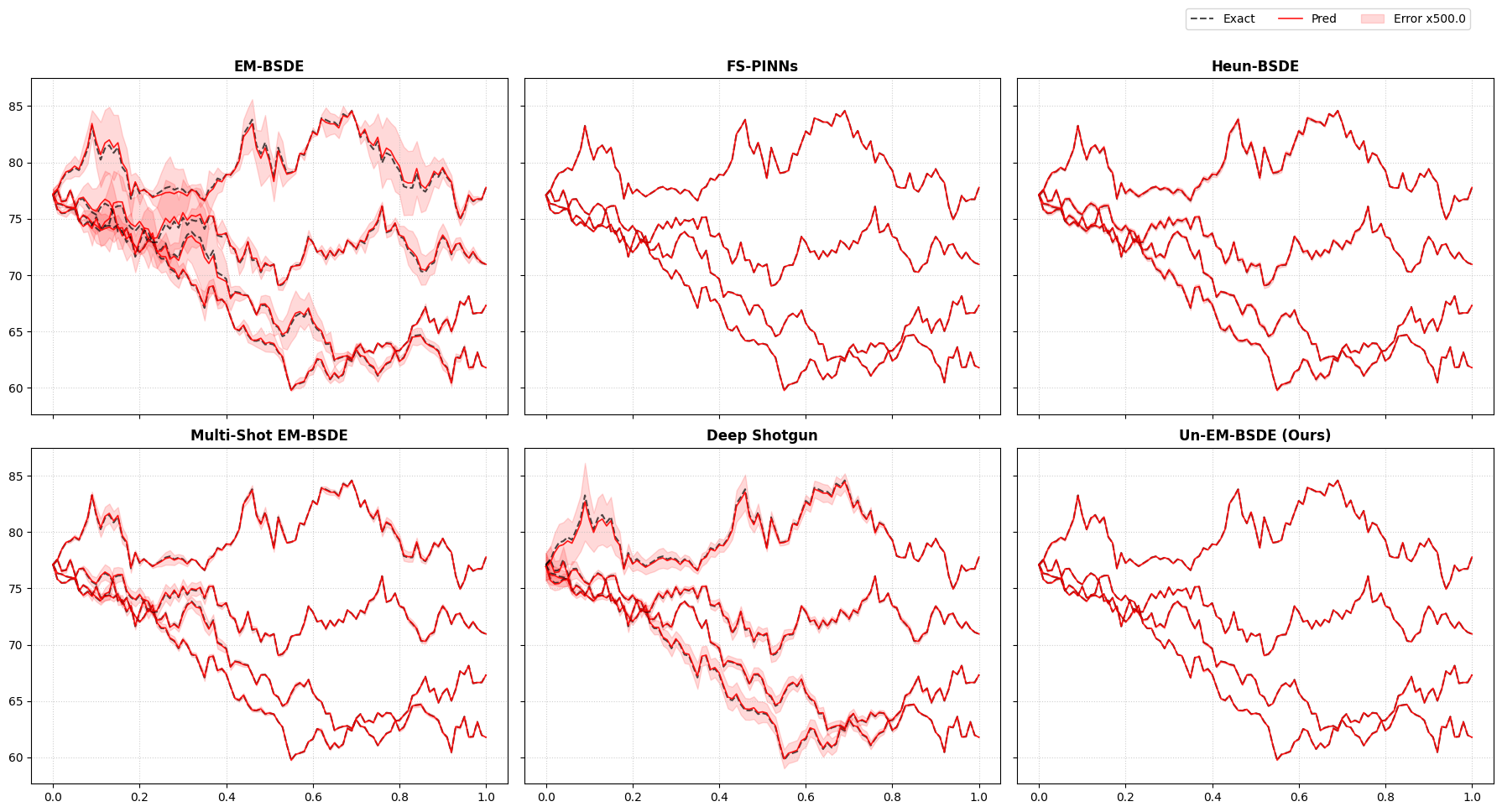}
  \caption{BSB($d=100$) path prediction with error band in the hard-constraint setting.}
  \label{fig:bsb_hard}
\end{figure}


\end{document}